\documentclass{article}
\usepackage{kr}

\usepackage{times}
\usepackage{soul}
\usepackage{url}
\usepackage[hidelinks]{hyperref}
\usepackage[utf8]{inputenc}
\usepackage[small]{caption}
\usepackage{graphicx}
\usepackage{amsmath}
\usepackage{amsthm}
\usepackage{booktabs}
\usepackage{algorithm}
\usepackage{amssymb}
\usepackage{mathtools}
\usepackage{hyperref}
\usepackage{bm} % for bold math symbols
\usepackage{enumitem} % for customized lists
\usepackage{microtype}
\usepackage{comment}
\usepackage{color}
\usepackage{tikz}
\usepackage{makecell}
\usepackage{algpseudocode}
\usepackage{multirow}
\usepackage{subcaption}   
\usepackage{wrapfig}

\newcommand{\redNote}[1]{#1}

\newcommand{\todo}[1]{{\color{blue}{TODO: #1}}}

\newcommand{\Prop}[1]{\mathsf{#1}}

\newcommand{\DenseForest}{\Prop{Forest}}
\newcommand{\DryVegetation}{\Prop{DryVegetation}}
\newcommand{\StrongWind}{\Prop{Wind}}
\newcommand{\LowHumidity}{\Prop{LowHum}}
\newcommand{\HighTemperature}{\Prop{HighTemp}} 
\newcommand{\RainedRecently}{\Prop{Rained}}
\newcommand{\LightningsFrequent}{\Prop{Lightnings}}
\newcommand{\LowHumanActivity}{\Prop{Isolated}}
\newcommand{\PowerLinesNearby}{\Prop{PowerLines}}

\newcommand{\Fuel}{\Prop{Fuel}}
\newcommand{\DryConditions}{\Prop{Dry}}
\newcommand{\Trigger}{\Prop{Trigger}} 
\newcommand{\WildFireRisk}{\Prop{WFRisk}}

\newtheorem{theorem}{Theorem}

\newtheorem{example}{Example}
\newtheorem{lemma}{Lemma}

\DeclareMathOperator{\onehot}{onehot}

\title{Logic of Hypotheses: \\
from Zero to Full Knowledge in Neurosymbolic Integration}

\author{%
Davide Bizzaro$^{1,2}$\and
Alessandro Daniele$^{2,3}$ \\
\affiliations
$^1$University of Padua, Padova, Italy\\
$^2$Fondazione Bruno Kessler, Trento, Italy\\
$^3$University of Bozen-Bolzano, Bozen, Italy\\
\emails
davide.bizzaro@phd.unipd.it,
alessandro.daniele@unibz.it
}
\begin{document}

\maketitle

\newif\ifarxiv
\arxivtrue  % switch to false for conference submission

%\ifarxiv
%  \fancyhead{} % no headers on arXiv
%  \renewcommand{\headrulewidth}{0pt} % remove the top line
%\fi

\begin{abstract}

Neurosymbolic integration (NeSy) blends neural‐network learning with symbolic reasoning. The field can be split between methods injecting hand-crafted rules into neural models, and methods inducing symbolic rules from data. We introduce \emph{Logic of Hypotheses} (LoH), a novel language that unifies these strands, enabling the flexible integration of data-driven rule learning with symbolic priors and expert knowledge. LoH extends propositional logic syntax with a \emph{choice operator}, which has learnable parameters and selects a subformula from a pool of options.  Using fuzzy logic, formulas in LoH can be directly compiled into a differentiable computational graph, so the optimal choices can be learned via backpropagation. This framework subsumes some existing NeSy models, while adding the possibility of arbitrary degrees of knowledge specification. Moreover, the use of Gödel fuzzy logic and the recently developed Gödel trick yields models that can be discretized to hard Boolean-valued functions without any loss in performance. We provide experimental analysis on such models, showing strong results on tabular data and on two NeSy tasks with a perceptual component.%\footnote{Code and supplementary materials are available at: \\ \href{https://dbizzaro.github.io/LoH/}{dbizzaro.github.io/LoH}}

\end{abstract}

\section{Introduction}
\label{sec:introduction}
%Integrating symbolic reasoning and neural networks %(NeSy) is an active research field aiming at %combining the interpretability and structure of %logic-based methods with the adaptability and %robustness of neural methods. Classical NeSy %approaches assume availability of explicit prior %knowledge, restricting their applicability to %domains where such symbolic rules are known.
%
%Recently, research has focused on learning symbolic %knowledge from data, inspired by Inductive Logic %Programming (ILP). Some advanced methods can %simultaneously learn symbolic rules and perception-%to-symbol mappings. However, these methods remain %limited in flexibility, typically fixed either %toward knowledge-driven or data-driven extremes.
%
%To address this limitation, we propose the %\textit{Logic of Hypothesis (LoH)}, a new logical %formalism introducing a flexible choice operator. %This allows expressing varying degrees of symbolic %knowledge explicitly, ranging from complete prior %knowledge to none. Leveraging Gödel fuzzy logic and %the Gödel Trick, LoH ensures discrete yet %differentiable hypothesis selection, making it %suitable for integration within neural networks.

Neurosymbolic integration (NeSy) tries to combine the symbolic and sub-symbolic paradigms. The aim is to retain the clarity and deductive power of logic while leveraging the learning capabilities of neural networks \cite{besold2021neural,MARRA2024104062}. 
%to exploit the interpretability and potential for reasoning of logical rules together with the statistical learning capabilities of neural networks \cite{besold2021neural,MARRA2024104062}. 
In many NeSy approaches, such as DeepProbLog \cite{manhaeve2018deepproblog} and LTN \cite{badreddine2022logic}, domain experts provide prior knowledge in the form of logic formulas, which the neural model uses as a bias or as constraints. While this strategy has proven effective, it presupposes that high‑quality rules are readily available. On the other hand, other NeSy methods have approached the learning of symbolic knowledge from data in the fields of rule mining \cite{dr-net,net-dnf,mllp} and inductive logic programming \cite{evans2018learning,rocktaschel2017end}.
Further, some advanced methods can simultaneously learn symbolic rules and perception-to-symbol mappings \cite{satnet,daniele2022deep,pmlr-v202-barbiero23a}, thereby grounding symbols to raw data while simultaneously discovering the logical structure. %Orthogonally, ILP systems \cite{evans2018learning,rocktaschel2017end} can accept user-provided background rules and induce new ones, but cannot be tailored at a \emph{propositional} level.

These research threads occupy opposite ends of a spectrum: knowledge‑\emph{injection} versus rule‑\emph{induction}.
However, they typically lack the flexibility to handle intermediate scenarios in which a \emph{partial logical structure} is supplied and the missing parts must still be learned. Such situations arise, for example, when existing prior knowledge must be revised or completed, or when learned rules are required to respect specific syntactic templates (e.g., CNF, DNF, Horn clauses, fixed‑length clauses) for which the model was never programmed. Addressing this flexible middle ground remains an open challenge for NeSy methods.
%However, these methods remain limited in flexibility, typically falling on the two knowledge-driven and data-driven extremes. It means that it would be difficult to adapt them to the cases in which a partial logical structure is given and only some parts are to be learned. 
%As a result, they are not meant to have the flexibility to  work on intermediate scenarios in which a \emph{partial} logical structure is given and only some parts are to be learned. For example, when some prior knowledge must be revised and/or completed, or when rules must be constrained to specific syntactic forms for which the model was not programmed. Addressing this flexible middle ground remains an open challenge for NeSy methods.

%In this paper, we propose the \emph{Logic of Hypotheses} (LoH), a new language that provides a unified framework for both extremes: from purely data-driven rule learning (zero prior knowledge) to purely knowledge-driven inference (full prior knowledge), and everything in between.
In this paper, we propose the \emph{Logic of Hypotheses} (LoH), a new logical formalism introducing the \emph{choice operator}, which learns to select a subformula from a set of candidates. 
Such language can be used to produce neural networks with varying degree of symbolic bias, ranging from complete prior knowledge to none.  
%Moreover, formulas in LoH can be directly compiled into a parametrized computational graph.
%Leveraging Gödel fuzzy logic, it ensures discrete yet differentiable hypothesis selection, that can be learned with Backpropagation and the Gödel Trick \cite{daniele2025noiserescueescapinglocal}. 
In this way, LoH offers a single, principled learning framework that can adapt to the amount and form of prior knowledge available.
Our main contributions are:
\begin{itemize}
    \item \textbf{A novel language (LoH)} that extends propositional logic syntax with a \emph{choice operator}, %enabling partial or complete specification of logical knowledge. 
    making it possible to leave parts of a formula underspecified. This allows to represent a hypothesis space $\mathcal{H}$ of formulas, in a flexible and compact way.  %The language is extremely flexible, allowing to represent any finite set of propositional formulas.
    %\item \textbf{A generality result}, showing that any finite set of propositional formulas) can be captured by a single LoH formula.
    \item \textbf{A compilation procedure} producing a differentiable computational graph from any LoH formula $\Phi$, allowing for the learning of a data-fitting logical formula among those in the hypothesis space represented by $\Phi$. We employ the Gödel trick \cite{Godel_trick}, a newly proposed stochastic variant of Gödel logic. Thanks to this choice, our approach can directly learn discrete functions through backpropagation.
    Moreover, the computational graph can be stacked on top of a neural network, allowing the end-to-end learning of symbolic rules alongside perception-to-symbol mappings.
    %Differently from other NeSy method, the discretized Boolean outputs of the extracted formula and of the differentiable model always coincide.
    \item \textbf{A unifying viewpoint} of NeSy paradigms. 
    The general neural layers of rule-inducing NeSy models like \cite{mllp,payani2019inductivelogicprogrammingdifferentiable} can be seen as the compilation of particular LoH formulas using product fuzzy logic. On the other hand, the full injection of prior knowledge can be obtained by simply avoiding the usage of the choice operator in a LoH formula.
    %LoH formulas without choice operators can be used to inject fully-provided logical rules into a neural architecture. 
    However, LoH is not limited to those two extremes and allows to construct models for many different intermediate situations (see Section~\ref{sec:unifying}). %like: selecting a subset of given rules or knowledge sources, selecting exactly one rule from each set, forming horn clauses, selecting candidate variables for missing conditions, etc. 
\end{itemize}

Code and supplementary materials are available at: \\ \href{https://dbizzaro.github.io/LoH/}{dbizzaro.github.io/LoH}

%The remainder of the paper is organized as follows:
%we first review related work in Section~\ref{sec:related}, and the background on fuzzy logic and the Gödel trick on Section~\ref{sec:background}. Then, we define LoH formally in Section~\ref{sec:loh}, discuss the differentiable compilation in Section~\ref{sec:compilation}, and provide theoretical insights in Section~\ref{sec:theoretical}. Section~\ref{sec:unifying} illustrates, with examples, how LoH encompasses a broad range of NeSy settings. %providing examples of LoH formulas and empirically evaluating them on artificial data. 
%Finally, sections~\ref{sec:classification} and \ref{sec:visualttt} report experimental results on real classification tasks, validating empirically the training of rule-inducing models  with the Gödel trick. 
%In summary, LoH offers a single, principled framework that can \emph{adapt to the amount and form of prior knowledge available}. It closes the gap between rule-provided and rule-induced NeSy methods, paving the way for systems that can seamlessly revise, extend, or respect given rules as needed. The remainder of the paper is organised as follows: Section 2 formalises the LoH syntax and semantics; Section 3 presents the differentiable compilation algorithm; Section 4 positions LoH within existing NeSy taxonomies; Section 5 reports experimental results; and Section 6 concludes with future directions.
\section{Related Works}
\label{sec:related}

%Neurosymbolic integration (NeSy) has historically advanced along two complementary directions: (i) \emph{knowledge injection}, where expert‑provided logical formulas actively constrain the learning and/or inference process, and (ii) \emph{rule induction}, where symbolic structure is discovered automatically from data. During the past decade these directions have spawned a rich repertoire of methods ranging from soft‑constraint regularizers to differentiable logic programs, rule‑learning neural networks, and continuous relaxations of classical symbolic frameworks.

\paragraph{Logics on top of Neural Predicates.}  
%Approaches such as \emph{Semantic‑Based Regularization} (SBR) \cite{diligenti2017semantic} and \emph{Logic Tensor Networks} (LTN) \cite{badreddine2022logic} translate given first-order formulas into differentiable penalties using fuzzy logic semantics, then add the penalties to the loss. \emph{Semantic Loss} \cite{xu2018semantic} follows a similar philosophy, but compiles propositional constraints into Sentential Decision Diagrams (SDDs) \cite{darwiche2011sdd} to obtain an exact, differentiable measure of logical violation of the given formulas.
Approaches such as \emph{SBR} \cite{diligenti2017semantic}, \emph{LTN} \cite{badreddine2022logic} and Semantic Loss \cite{xu2018semantic} translate logical knowledge into differentiable penalties added to the loss.  In contrast, abductive methods \cite{dai2019abductive,tsamoura2021neurolog,huang2021fast} let a symbolic model find labels for the neural part consistent with the provided knowledge. This enables logical reasoning also at inference time, yet the symbolic program remains user‑supplied rather than learned. 
Similarly, \emph{DeepProbLog} \cite{manhaeve2018deepproblog}, \emph{DeepStochLog} \cite{winters2022deepstochlog}, and \emph{NeurASP} \cite{yang2020neurasp} enrich logical solvers with neural predicates whose outputs are treated as probabilities. %\redNote{Neuro-fuzzy networks (NFNs) learn fuzzy rule bases but, even in variants that perform structural identification, they retain a predetermined, non-flexible layer schema. Moreover, NFNs aggregate rule contributions via fuzzy inference rather than searching a symbolic hypothesis space to select a single discrete formula like our method.} 
%The \emph{Gödel Trick} \cite{Godel_trick} makes formulas differentiable via Gödel semantics, and add noise to avoid local minima. Optimizing with backpropagation can then be interpreted as a (discrete) local search algorithm for SAT solving. 
Here what is learned sits on neural modules rather than on alternative rule bodies.

\emph{Neuro-fuzzy networks} like \cite{jang1993anfis}, \cite{lazyPop} and \cite{GSETSK} are interpretable models where a neural network structure directly coincides with a fuzzy rule base. Typically they concentrate on parametric identification under fixed rules. Even when logic is induced \cite{NFN-survey}, their rule-based architecture is suited for a DNF-like format, which contrast the expressivity and flexibility of LoH.

%\redNoteNew{
%In probabilistic logic programming and statistical relational learning, random variables, probabilistic facts and annotated disjunctions could be seen as performing ``choices''. These are different 
%random variables taking values on a finite domain could be seen as ``
%probabilistic facts and annotated disjunctions define distributions over possible worlds.
%Regarding the ``choice'' semantics, in probabilistic logic programming and statistical relational learning, probabilistic facts and annotated disjunctions define distributions over possible worlds. 
In probabilistic logic programming, probabilistic facts and annotated disjunctions could be employed to represent ``choices''. However, these define 
%In probabilistic logic programming, one may think random variables taking values in a finite domain as ``choices''. This concept of ``choice'' lives in the semantics of \emph{interpretations}, with probabilistic facts and annotated disjunctions defining 
distributions over possible worlds, and extracting a single crisp program typically requires probabilities to collapse to $0/1$, as picking always the most likely local choice is only a heuristic. 
In LoH, instead, choices live in the space of \emph{hypotheses}, and a compiled LoH formula act as one function in the hypothesis space, making binary deterministic yes/no predictions: the one obtained substituting each choice operator with the subformula with weight $>0.5$.
%In LoH, instead, training searches for one global deterministic formula in the hypothesis space: the one obtained by substituting each choice operator with the subformula with weight $>0.5$. By construction, the resulting crisp formula exactly matches the trained network, and 
And learning can be seen as a differentiable combinatorial local search procedure under the \emph{Gödel Trick} \cite{Godel_trick}. This makes LoH especially suited for rule learning and deterministic decision modeling, whereas probabilistic logic programs are naturally tailored to probabilistic inference under uncertainty.
%}

On the rule-inducing side, \emph{SATNet} \cite{satnet} embeds a smoothed MaxSAT layer inside a neural network, jointly optimizing clause weights and perception. Subsequent work exposed limitations with unsupervised grounding \cite{chang2020assessing}, partially alleviated in \cite{topan2021techniques}.  %Instead, \emph{Meta-abd} \cite{dai2021abduction} combines inductive logic programming for learning the logical part, with abduction for training the neural mapping from perceptions to predicates. Finally,
\emph{DSL} \cite{daniele2022deep} directly learn symbolic rules from data alongside the perception‑to‑symbol mappings, but the symbolic part is only a lookup table.

%When only raw data are labelled, abductive approaches jointly induce symbolic concepts and theories. \emph{NeuroLog} \cite{tsamoura2021neurolog}, \emph{Abductive Learning} (ABL) \cite{dai2019abductive}, \emph{ABLSim} \cite{huang2021fast}, and \emph{MetaAbd} \cite{dai2021abduction} integrate a reasoning layer that forces the perception module to emit symbols consistent with the evolving theory. The \emph{Apperception Engine} \cite{evans2021apperception} follows a similar strategy but relies on fixed vision backbones. \emph{Deep Symbolic Learning} (DSL) \cite{daniele2022deep} removes this bottleneck via end‑to‑end optimization, though it still relies on product -norm surrogates and auxiliary policy functions.

%\paragraph{Inductive logic programming.}   Modern ILP NeSy methods learn first-order rules via gradient descent. dILP~\cite{evans2018learning} assigns soft weights to a forward-chaining template. Neural Logic Networks (NLN) \cite{payani2019inductivelogicprogrammingdifferentiable} introduce logical neurons using product fuzzy logic---the same neurons later reused by MLLP \cite{mllp} for propositional rule learning---to synthesize first-order programs. NTP \cite{rocktaschel2017end} replaces hard unification with differentiable embedding similarity.

\paragraph{Neural Networks with Soft Gates.}
%\paragraph{Propositional rule‑learning networks.}  
Many recent NeSy learners devise neurons that perform a continuous relaxation of the \texttt{AND} and \texttt{OR} operations, with learnable weights acting as soft gates. These neurons are placed into layers alternating the two operations, while the \texttt{NOT} operation is obtained by doubling the inputs---juxtaposing each input value with its negation. 
Models like these are typically used to learn propositional rules on binarized tabular data \cite{dr-net,rl-net,net-dnf,r2n,perreault2025neural}. Among these, \emph{Multi‑Layer Logical Perceptron (MLLP)} \cite{mllp} is a state-of-the-art model, using product fuzzy logic operations. However, like all the others, it does not guarantee that the extracted rules---which it calls \emph{Concept Rule Sets (CRS)}---have the same accuracy as the neural model. 
%\emph{Decision Rule Network} (DR‑Net) \cite{dr-net}, its multi‑class extension \emph{Rule Learning Network} (RL‑Net) \cite{rl-net}, \emph{Net-DNF} \cite{net-dnf},  \emph{Relational Rule Network} (R2N) \cite{r2n}, and \emph{Multi‑Layer Logical Perceptron} (MLLP) \cite{mllp}.
%Their main individual differences are on how the soft logical operations are implemented, with the last one being the only one using sound fuzzy logic operations. 
%\cite{mllp} also introduces a different name for the hard logical rules obtained by discretizing the weights with a threshold, calling them \emph{Concept Rule Sets (CRS)}. Indeed, none of the above works guarantee that the extracted rules have the same accuracy as the neural model. 
Instead, we propose to use Gödel fuzzy logic, which allows a lossless extraction. 

Soft logical gates are also employed on binary/ternary neural networks \cite{deng2018gxnor}, which are typically used for the efficiency of the quantized networks at inference, rather than the logical semantics. In particular, \emph{Differentiable Logic Networks (DLN)} \cite{petersen2022deep,petersen2024convolutional,yousefi2025mind}  keep a sparse fixed wiring with nodes having at most two parents, and learn which binary Boolean operation each node should execute. Instead, LoH does the opposite: the logical operations are fixed, and the learnable gates decide which branch is selected. This allows LoH to yield more readable formulas (as \emph{DLNs} rely on deeply nested structures employing 16 different logical operators), and offers a more straightforward path for incorporating prior knowledge.

\paragraph{Inductive Logic Programming (ILP).} Also modern NeSy methods in ILP, such as $\partial$\emph{ILP}~\cite{evans2018learning}, \emph{NTP} \cite{campero2018logical} and \emph{NeurRL} \cite{gao2025differentiable}, use neurons performing a soft version of the logical operations, thus learning first-order rules via gradient descent. %Examples are \emph{dILP}~\cite{evans2018learning} and \emph{Neural Theorem Proving} (NTP). \cite{rocktaschel2017end,campero2018logical}  and \emph{Neural Logic Networks} (NLN) \cite{payani2019inductivelogicprogrammingdifferentiable}. The last one is the only one using sound fuzzy logic operations, and it was also the source of the neurons used by MLLP in \cite{mllp}. %In the following, we focus on propositional tasks, leaving the application of LoH to first-order logic and ILP as future work. %However, there is nothing preventing such usage.
Of particular interest are \emph{Logical Neural Networks} ~\cite{riegel2020logical}, which compile formulas into neural networks with weighted  Łukasiewicz operators, but rely on constrained optimization, %(e.g., Frank-Wolfe \cite{frank1956algorithm}), 
which hampers scalability. Moreover, assigning ``importances'' to the subformulas, instead of choosing one among the candidates as in LoH, limits the possibility to extract hard rules without loss in accuracy. 

\redNote{In principle, ILP is a natural paradigm to span the full ``zero to full knowledge'' spectrum, since it can exploit background knowledge while inducing missing rules. However, as discussed in Section \ref{sec:unifying}, current differentiable ILP systems do not implement this possibility: expressing intermediate regimes (e.g., selecting subsets of a given pool of rules or even just providing formulas not following the system templates) requires substantial ad-hoc machinery, and it is less practical and less transparent than LoH.}

\section{Background}
\label{sec:background}

%\subsection{Propositional and Fuzzy Logic}
%\paragraph{Propositional Logic}
Classical propositional logic builds formulas from propositional variables using negation ($\neg$), conjunction ($\land$) and disjunction ($\lor$).\footnote{For simplicity, we do not consider implication ($\to$) and iff ($\leftrightarrow$). However, there is nothing preventing their use, if associated with a consistent fuzzy semantics. For example, we may consider \emph{material} implication, substituting $\phi \to \psi$ with its Boolean equivalent $\neg \phi \lor \psi$.} An interpretation assigns to each variable the Boolean value true ($1$) or false ($0$), and extends recursively: the interpretation of $\neg \phi$ flips the value of $\phi$, $\phi \land \psi$ returns the conjunction (\texttt{AND}) of the two values, and $\phi \lor \psi$ returns the disjunction (\texttt{OR}). 
%\paragraph{Fuzzy Logic.}
Fuzzy logics relax the interpretations' truth values to the real unit interval $[0, 1]$, interpreting connectives with \emph{t-norms} (for $\land$) and \emph{t-conorms} (for $\lor$). This relaxation brings differentiable operations, allowing gradient-based optimization. Common fuzzy logics include Łukasiewicz, Product, and Gödel.
Product logic has $t(x,y):= xy$ as t-norm (i.e., conjunction), and $s(x,y):= 1{-}(1{-}x)(1{-}y)= x{+}y{-}xy$ as t-conorm (i.e., disjunction). On the other hand, 
Gödel logic uses $\min$ and $\max$ for conjunction and disjunction, respectively. In both, the negation of $x$ corresponds to $1{-}x$.

Gödel logic stands out for its simplicity and its closer alignment to classical logic in terms of idempotency and distributivity.
Importantly, Gödel logic has the following property:
\begin{theorem}[Prop.~1~in~\cite{Godel_trick}]\label{th:binarization}
    For any Gödel interpretation $\mathcal{G}$, let $\mathcal{B}$ be the Boolean interpretation obtained rounding every fuzzy value in $\mathcal{G}$ with the thresholding function\footnote{The exclusion of $x = 0.5$, where negation would break the homomorphism property, is mostly a technicality. In practical settings, this exact value has measure zero in continuous-valued interpretations and does not affect the general applicability of the result.}
    \begin{equation}\label{eq:thresholding_function}
    \rho\colon [0,1]\setminus\{0.5\} \to \{0,1\}, \quad  x \mapsto \begin{cases}
    1 \ \text{ if } x>0.5 \\
    0 \ \text{ if } x<0.5
\end{cases}
\end{equation}
    meaning that $\mathcal{B}(v_i) = \rho(\mathcal{G}(v_i))$ for every propositional variable $v_i$. Then, $\mathcal{B}$ is always consistent with $\mathcal{G}$, i.e., $\mathcal{B}(\phi) = \rho(\mathcal{G}(\phi))$ for every formula $\phi$.
\end{theorem}

We can think at any logical formula as our model, with the truth values of the variables as inputs and the corresponding truth value of the formula as output. It is continuous if using fuzzy interpretations, and discrete if using classical Boolean ones.
The theorem means that discretizing the outputs of the continuous model is the same as working with the discrete one, on the discretized inputs.

%Fuzzy logics relax Boolean connectives to continuous operators so that gradient-based optimization can tune symbolic structures. Common families include Łukasiewicz, Product, and Gödel $t$-norms.
%A key challenge is balancing \emph{discreteness} (logical expressions, truth values in $\{0,1\}$) with \emph{differentiability} (the cornerstone of gradient-based optimization). Fuzzy logic is a common solution, approximating logical connectives by continuous functions.
%Gödel logic, which uses $\min$, $\max$, and $1{-}x$ (for conjunction, disjunction and negation, respectively), stands out for its simplicity and its closer alignment to classical logic in certain respects \cite{van2022analyzing}.
%Gödel logic is particularly appealing for rule learning because conjunction ($\min$) and disjunction ($\max$) act as hard selectors, closely mirroring crisp logic when weights are restricted to ${0,1}$. Recent analyses\cite{vankrieken2022godel} show that Gödel provides sparser gradients than Product, encouraging discrete solutions.

%\subsection{The Gödel Trick}
The main problem of using Gödel semantics is that its optimization can stall in shallow local minima. The Gödel Trick \cite{Godel_trick} counters this by adding noise to each parameter, turning the optimization into a stochastic local search while remaining gradient-based. %We briefly present here how the Gödel trick with categorical variables work, and refer to \cite{Godel_trick} for a more complete discussion. 
It is equivalent to the the classical Gumbel-max trick \cite{gumbel1954statistical}, and it works by storing as parameters the logits of the fuzzy weights. Then, for every step of training, random noise is sampled and added to the logits. This is done in the forward pass, before applying the sigmoid function producing the fuzzy weights. %No noise is added at inference.
%See \cite{Godel_trick} for a more complete discussion, and a proof of its equivalence to the classical \emph{Gumbel-max trick} \cite{gumbel1954statistical}
%see .
%With logistic noise, GT reduces to a differentiable variant of the Gumbel-Max/Gumbel-Softmax re-parameterisation, ensuring that gradients flow while the final decisions are still one-hot and exact.

\section{A Language for Hypothesis Spaces}
\label{sec:loh}
We introduce \emph{Logic of Hypotheses} (LoH) first as a language for expressing hypothesis spaces (i.e., sets) of formulas in compact way.
Syntactically, LoH extends propositional logic, adding a new \emph{choice operator} $[\cdot]$ that can take as input any finite number of formulas: 
%Let $\mathcal{V}$ be a set of propositional variables (e.g., $A, B, \ldots$). LoH augments standard logic with a \emph{choice operator}:
\begin{equation*}
  \Phi ::=\, \top \;\big|\; \bot \;\big|\; v \;\big|\; \lnot \Phi_1 \;\big|\; \Phi_1 \lor \Phi_2 \;\big|\; \Phi_1 \land \Phi_2 \;\big|\; [\Phi_1,\dots, \Phi_n]
\end{equation*}
where $n$ can vary among the positive integers and $v\in \mathcal{V}$ represents the propositional variables. %The subformula $[\Phi_1, \Phi_2,\dots,\Phi_n]$ denotes a \emph{choice} among the enclosed candidate formulas. 
Semantically, a LoH formula $\Phi$ represents an entire set of classical propositional formulas $\mathcal{H}(\Phi)$, each obtained by selecting exactly one subformula per choice operator. %The classical semantics of propositional logic will come back when discussing the ``compilation'' of LoH formulas, but for now an LoH formula expresses a set propositional ones, so 
\begin{example}
    The LoH formula $\Phi \coloneqq [a,b]\land [c, d] \land \neg e$ has hypothesis space
    \[
    \mathcal{H}(\Phi) \coloneqq \{\, a \land c \land \neg e,\, a \land d \land \neg e,\, b \land c \land \neg e,\, b \land d \land \neg e\,\}
    \]
\end{example}
In general, the set $\mathcal{H}(\Phi)$ is obtained by applying inductively the following \begin{comment} substitutions to $\Phi$: 
\begin{enumerate}[label=R\arabic*:]
    \item the set $\{v\}$ substitutes the propositional variables $v$;
    \item the set $\{\neg \phi \mid \phi \in \Phi_1\}$  substitutes the subformulas $\neg \Phi_1$;
    \item the set $\{\phi \land \psi \mid \phi \in \Phi_1, \psi \in \Phi_2\}$ substitutes the subformulas $\Phi_1 \land \Phi_2$;
    \item the set $\{\phi \lor \psi \mid \phi \in \Phi_1, \psi \in \Phi_2\}$ substitutes the subformulas $\Phi_1 \lor \Phi_2$;
    \item the set $\bigcup_{i=1}^n \Phi_i$ substitutes the subformulas $[\Phi_1,\dots,\Phi_n]$. 
\end{enumerate}
\end{comment}
operations: 
\begin{enumerate}[leftmargin=1.2cm, label=R\arabic*: \,]
    \item $\mathcal{H}(v) = \{v\}$;\;  $\mathcal{H}(\top) = \{\top\}$;\; $\mathcal{H}(\bot) = \{\bot\}$;
    \item $\mathcal{H}(\neg \Phi_1) = \{\neg \phi \mid \phi \in \mathcal{H}(\Phi_1)\}$;
    \item $\mathcal{H}(\Phi_1 \land \Phi_2) = \{\phi \land \psi \mid \phi \in \mathcal{H}(\Phi_1),\, \psi \in \mathcal{H}(\Phi_2)\}$;
    \item $\mathcal{H}(\Phi_1 \lor \Phi_2) = \{\phi \lor \psi \mid \phi \in \mathcal{H}(\Phi_1),\, \psi \in \mathcal{H}(\Phi_2)\}$;
    \item $\mathcal{H}([\Phi_1,\dots,\Phi_n]) = \bigcup_{i=1}^n \mathcal{H}(\Phi_i)$.
\end{enumerate}

\begin{example}
Let's unfold the step-by-step procedure for producing $\mathcal{H}([a, [b, c]] \land \neg [c, d])$:
%the hypothesi space corresponding to the following LoH formula:
%\[
%    [\neg a, [b, c]] \land [c, d]
%\]
\begin{enumerate}[leftmargin=1.15cm]
\item[(R1)] $\mathcal{H}(a) {=} \{a\}$; $\mathcal{H}(b) {=} \{b\}$; $\mathcal{H}(c) {=} \{c\}$; $\mathcal{H}(d) {=} \{d\}$;
\item[(R5)] $\mathcal{H}([b,c]) {=} \{b, c\}$; $\mathcal{H}([c, d]) {=} \{c, d\}$;  
\item[(R5)] $\mathcal{H}([a, [b,c]]) {=} \{a, b, c\}$;
\item[(R2)] $\mathcal{H}(\neg [c, d]) {=} \{\neg c, \neg d\}$; 
\item[(R3)] $\mathcal{H}([a, [b, c]] \land \neg [c, d]) {=} \{a \land \neg c,\, a \land \neg d,\, b \land \neg c, \\ %\hphantom{\mathcal{H}([a, [b, c]] \land \neg [c, d]) {=} \{
%\hphantom{\quad\quad} 
b \land \neg d,\, c \land \neg c,\, c\land \neg d\}$
    %\item[(R1) ] $[\{a\}, [\{b\}, \{c\}] \land \neg [\{c\}, \{d\}]$ 
    %\item[(R5) ] $[\{a\}, \{b, c\}] \land \neg \{c, d\}$
    %\item[(R5)+(R2) ] $\{a, b, c\} \land \{\neg c, \neg d\}$
    %\item[(R3) ] $\{a \land \neg c, a \land \neg d, b \land \neg c, b \land \neg d, c \land \neg c, c\land \neg d\}$
\end{enumerate}

\begin{comment}
\[
    \neg [a, [b, c]] \lor [a, c \land d] \lor \neg d
\]
\begin{enumerate}
    \item $\neg [\{a\}, [\{b\}, \{c\}]] \lor [\{a\}, \{c\} \land \{d\}] \lor \neg \{d\}$
    \item $\neg [\{a\}, \{b, c\}] \lor [\{a\}, \{c \land d\}] \lor \{\neg d\}$
    \item $\neg \{a, b, c\} \lor \{a, c \land d\} \lor \{\neg d\}$
    \item $\{\neg a, \neg b, \neg c\} \lor \{a \lor \neg d, c \land d \lor \neg d\}$
    \item $\{\neg a \lor a \lor \neg d, \neg a \lor  c \land d \lor \neg d, \neg b \lor a \lor \neg d, \neg b \lor  c \land d \lor \neg d, \neg c \lor a \lor \neg d, \neg c \lor  c \land d \lor \neg d\}$
\end{enumerate}
\end{comment}
\end{example}

In neural networks, the output of a hidden neuron is usually fed to multiple neurons of the subsequent layer.  Similarly, in LoH, we may want to use the same ``choice'' of a subformula in multiple places. This can be solved by %having a set of \emph{choice variables} in addition to the propositional variables. These are assigned to each instantiation of the choice operator, and can be used as a  
defining a placeholder for a sub-formula, and use it in multiple parts of the main formula. The algorithm for producing the hypothesis space corresponding to an LoH formula with such placeholders is reported in Appendix \ref{app:choice_variables}.
\begin{example}
    The LoH formula $[a, b] \land [a, b]$ has hypothesis space $\{a \land a, a\land b, b \land a, b\land b\} \equiv \{a, a\land b, b\}$. On the other hand, the LoH formula $\phi \land \phi$ with $\phi \coloneqq [a,b]$ has hypothesis space $\{a \land a, b\land b\} \equiv \{a, b\}$. 
    In $[a,b]\land[a,b]$ the two choices are independent, whereas $\phi\land\phi$ introduces just one choice and reuses the selected subformula twice. %The way to think about it is that placeholders implement weight sharing: in $[a,b]\land[a,b]$, the weights of the two choice operators are different, while in $\phi\land\phi$ they are shared.
\end{example}

\begin{comment}
Many times, we may want to use the same ``choice'' in different parts of a formula/architecture/hypothesis space (think for example of the output of a neuron being fed to all the neurons in the next layer). For this reason, we add assignments to the language:
\begin{equation}
    x \coloneqq [p_1, \dots, p_n]
\end{equation}
means that, whenever using $x$ in a formula, it is $[p_1, \dots, p_n]$ with always the same weights. 
More generally, we may give assignments $x\coloneqq F$ also to complex formulas $F$, so that every time we use $x$, all the choices inside $F$ are fixed.

Indeed, it is important to notice that the graph representing a formula could be a DAG rather than a tree. For example, 
\[
    (a \land [c, d]) \lor (b \land [c, d])
\]
would be interpreted as a tree, with hypothesis space 
\[
    \{(a \land c)\lor (b \land c), (a\land d)\lor (b \land c), (a \land c)\lor (b \land d), (a\land d)\lor (b \land d)\}
\]
but
\[
    (a \land e) \lor (b \land e), \quad e\coloneqq [c, d]
\]
would be interpreted as a DAG, with hypothesis space 
\[
    \{(a \land c)\lor (b\land c), (a \land d)\lor (b\land d)\}
\]
\end{comment}

%\paragraph{Expressiveness.}

%\begin{theorem}
%    For any finite set $H$ of propositional formulas, there exists a formula $\phi$ in LoH such that $\mathcal{H}(\phi) = H$.
%\end{theorem}
%\begin{proof}
%    The trivial formula $\phi_H \coloneqq [h]_{h\in H}$ is such that .
%\end{proof}
%\paragraph{Expressivity.}
Notice that LoH is flexible enough to encode \emph{any} finite set of propositional formulas $\{h_1, \dots, h_n\}$. Indeed,  $[h_1,\dots, h_n]$ represents exactly that space, even if  more compact representations---whose compilations will require less parameters---may be possible. Even for a fixed hypothesis space, LoH is flexible enough to provide formulas biasing the search process in different ways. For example, when compiled, both $[a, [b,c]]$ and $[a, a, b, c]$ are more biased towards choosing $a$ than $[a,b,c]$. 
%Hence, LoH is expressive enough to capture a wide range of knowledge, from single rules to entire rule sets.
%LoH can capture arbitrary sets of formulas. 
%With the choice operator, the space of possible discrete hypotheses is combinatorial, making LoH highly expressive. The differentiable framework allows searching this space via gradient-based optimization.

\section{To Differentiable Computational Graphs}
\label{sec:compilation}

In the preceding section, we presented LoH as a language for expressing hypothesis spaces of logical formulas. We now show how the same LoH formulas can be turned into supervised machine learning models searching in those hypothesis spaces. The search will be done by gradient descent with backpropagation, so we need to compile LoH formulas into differentiable computational graphs. 

The first step is to introduce a weight $w_i \in [0,1]$ for every candidate subformula $\Phi_i$ inside a choice operator. These are learnable and act as gates. Each choice operator is then converted to a propositional formula linking such weights to the respective subformulas. This can be done in two dual, practically interchangeable ways, which we call \emph{disjunctive}/\emph{conjunctive}  \emph{compilations}: 
\begin{flalign*}\label{eq:choice-operator}
    \text{\textit{Disjunctive Compilation}}\; && [\Phi_1,\dots,\Phi_n] \ &  \leadsto \ \ \bigvee_{i=1}^n \ \ \ \, w_i \land \Phi_i && & \\
    \text{\textit{Conjunctive Compilation}} && [\Phi_1,\dots,\Phi_n] \ &  \leadsto \ \ \bigwedge_{i=1}^n \ \neg w_i \lor \Phi_i && &
\end{flalign*}

Whichever of the two we use---more on this later---, we are left with a propositional formula with only the operators $\neg$, $\land$ and $\lor$. In order to have differentiable operations, we interpret them under a fuzzy semantics.  For example, with Gödel fuzzy logic, $\bigvee_{i=1}^n w_i \land \Phi_i$ becomes $\max_{i=1,\dots,n}(\min(w_i, \Phi_i))$ and  $\bigwedge_{i=1}^n \neg w_i \lor \Phi_i$ becomes $\min_{i=1,\dots,n}(\max(1-w_i, \Phi_i))$.
The final piece is to design the weights in such a way that they can take continuous values in the interval $[0,1]$, while allowing to extract the discrete selection of a candidate subformula for every choice operator.

%It can be considered as an operator transforming inputs ($p_1,\dots,p_n$ truth values) into outputs, when the weights are fixed to some values. On the other hand, the set of formulas that can be obtained changing the values of the weights describe an hypothesis space. 

\paragraph{Design of the weights.} 
Let us first consider the case in which the weights are binary, i.e., each $w_i$ can only have value $0$ or $1$. Then, the formulas above can be simplified to equivalent ones, recalling that $0 \land \Phi \equiv 0$, $1 \land \Phi \equiv \Phi$, $0 \lor \Phi \equiv \Phi$ and $1 \lor \Phi \equiv 1$, for any formula $\Phi$. It follows that the disjunctive (resp.\ conjunctive) compilation is equivalent to the disjunction (resp.\ conjunction) of the subformulas with weight $1$. So if we impose that one weight $w_i$ is $1$ and the remaining are $0$, then both the conjunctive and the disjunctive compilation become equivalent to the single ``chosen'' subformula (the one with $w_i =1$). This is exactly the condition we want after discretizing the weights:  %$w_i$ for extracting hard interpretable rules living in the hypothesis spaces denoted by the LoH formulas. 
the discrete selection of a \emph{single} candidate from each choice operator.%\footnote{This is true independently of whether we use disjunctive or conjunctive compilation. Indeed, they both have the same purpose of selecting \emph{one} subformula among the candidates, and---when discretized---returning exactly its truth value. The reasons for preferring one or the other will be exposed later.} 

The simplest way to discretize the weights is to use the thresholding function $\rho$ defined in \eqref{eq:thresholding_function}. Hence, for any tuple of weights $(w_1, \dots, w_n)\in [0,1]^n$ associated to a choice operator $[\Phi_1,\dots,\Phi_n]$, we want to impose  that $w_{i}>0.5$ for one and only one $i$.
%to impose by design the following conditions:
%\begin{itemize}
%    \item they must live in the interval $[0,1]$
%    \item one and only one among them is $\geq 0.5$
%\end{itemize}
%The first condition is for the consistency with the fuzzy logic semantics. It is solved by deriving the weights $w$ from %passing \emph{preactivation} values $z\in \mathbb{R}$ to some squeezing function $f\colon \mathbb{R} \to [0,1]$. We consider the sigmoid function $\sigma(z) = \frac{\exp(z)}{1+\exp(z)}$
%the application of the sigmoid function to some \emph{logit} values  $z\in \mathbb{R}$: 
%\[w = \sigma(z) = \frac{\exp(z)}{1+\exp(z)}\]
%\redNote{Should we allow also different functions, like clipping?}
%The second condition is for allowing the extraction of a single discrete choice. This can be obtained in the following way: 
Instead of storing the weights $w_i$ directly, let us associate to each of them the actual learnable parameter $z_i$, which can take any real value. The differentiable operations for deriving the weights $w_i$ from these parameters are the following:
\begin{enumerate}
    \item To escape local minima in the optimization procedure, at each \emph{forward step of training}, add random noise to the parameters: $z_i' \coloneqq z_i + n_i$ with $n_i \sim Gumbel(0, \beta)$ and $\beta$ being an hyperparameter.
    \item Let $\bar z'$ be the mean of the two largest $z_i'$ values, and subtract it to each $z_i'$. By construction, 
    %In the real line, $\bar z'$ is the middle point between the largest $z_i'$ and the second largest, so 
    all points $z_i'$ but the largest lie on the left of $\bar z'$.\footnote{The probability of the two largest values coinciding is negligible, especially after adding the continuous Gumbel noise. Moreover, this happening would affect only the extraction of hard rules, not the optimization procedure (which would continue to change the parameters, eventually breaking the tie).} 
    Hence, one and only one among the shifted values $z_i'' \coloneqq z_i' - \bar z'$ is positive.
    \item Apply the sigmoid function: $w_i {\coloneqq} \sigma (z_i''/T){=}\frac{\exp(z_i''/T)}{1+\exp(z_i''/T)}$, with the temperature $T$ being an hyperparameter.
\end{enumerate}

The application of the sigmoid function ensures that the weights are in the interval $[0,1]$. Moreover, because of the second step, one and only one logit $z_i''$ is positive, so exactly one weight $w_i$ is greater than  $0.5$. This procedure for producing weights $w_i$ from the $z_i$'s is an adaptation of the Gödel trick with categorical variables \cite{Godel_trick}.%, and is further discussed in Appendix \ref{app:godeltrick}.

\paragraph{Disjunctive vs Conjunctive Compilation.} Both the disjunctive and the conjunctive compilation have the same purpose of selecting \emph{one} subformula among the candidates, and in general both work well. 
%But when is it better to use disjunctive compilation and when conjunctive? 
However, we suggest using the former when the choice operator is a term in a disjunction, and the latter when in a conjunction; the opposite if negated. For example, for $\neg [a,b]\land [b,c] \land ([c, d] \lor \neg [d, e])$, we suggest using disjunctive compilation for $[a,b]$ and $[c,d]$, and conjunctive compilation for $[b,c]$ and $[d,e]$. The theoretical and empirical motivations for this are discussed in the supplementary materials.

\paragraph{Robustness to Binarization.}
By design, it is always possible to binarize the weights of a model and extract the chosen subformula from each choice operator. The result is a propositional formula in the hypothesis space denoted by the LoH formula. If using Gödel fuzzy logic, then Theorem~\ref{th:binarization} guarantees that the outputs of the resulting propositional formula \emph{coincide} with the rounding of the outputs of the continuous model on \emph{every} possible input.  %\footnote{See Appendix \ref{app:counterexample} for a counterexample with product fuzzy logic.}  
Accuracy, confusion matrices, and any higher‑level deductive tasks are thus preserved from the continuous to the discrete model.  In this sense, Gödel logic offers lossless rule extraction (while Łukasiewicz and Product t-norms do not). Also notice that this is true at any point in training. So the entire training process---while being gradient-based---can be interpreted symbolically, through discrete changes. To our knowledge, no other rule-learning NeSy model achieves this.

%Applying the sign function $s(\cdot)$ to the final logits recovers a standard Boolean interpretation. Indeed, for subformula $\phi$ with logit $\ell \in \mathbb{R}$, $s(\ell)$ is $+1$ if $\ell > 0$ and $-1$ if $\ell < 0$, exactly matching classical propositional truth values (true/false).

%\paragraph{Discrete behavior from continuous training.}
%One key advantage of the Gödel logic approach is \emph{sparse gradients} \cite{van2022analyzing}. In each update step, exactly one path in the computational graph (the minimal or maximal branch) becomes active for backpropagation, which neatly mirrors local search behavior in SAT solving.

\section{LoH as a Unifying NeSy Framework}
\label{sec:unifying}

In this section, we demonstrate how LoH captures a wide range of existing NeSy frameworks, spanning fully-provided prior knowledge, partially known templates, and purely data-driven rule induction. 
To illustrate these settings, we instantiate them on a small synthetic task: a wildfire–risk assessment example problem. 
%A detailed description of the data generator, neural and symbolic architectures, and training protocol is deferred to Appendix \ref{app:wildfire}; here we only recall the ingredients that are needed to understand the different knowledge settings.

\paragraph{Experimental Setup: Wildfire Risk.}
%However, these two variables are not provided to the models at training or test time: the models only observe the image and the 7 environmental features, and must infer Forest and DryVegetation via a CNN $h_\theta$. The architecture therefore consists of (i) a generic CNN $h_\theta$ that maps each image to predictions of the latent concepts Forest and DryVegetation, and (ii) a logical component that combines these predicted visual concepts with the environmental features to produce the final WFRisk prediction. Crucially, the CNN is not pretrained and must learn this grounding solely through gradients backpropagated from the WFRisk supervision through the logic. 

We generated a dataset of $2048$ samples, each consisting of a simulated aerial image and a set of $7$ Boolean environmental features ($\StrongWind$, $\LowHumidity$, $\HighTemperature$, $\RainedRecently$, $\LightningsFrequent$, $\LowHumanActivity$, $\PowerLinesNearby$). The task is to predict a binary $\WildFireRisk$ label.
\redNote{To generate the images and the ground-truth $\WildFireRisk$ labels, two additional Boolean variables ($\DenseForest$ and $\DryVegetation$) are used. Indeed, the ground-truth labels are generated as the conjunction of the following three rules:
\begin{subequations}\label{eq:wildfire}
\begin{align}
\Fuel \;&:=\; \DenseForest \,\lor\, (\DryVegetation \land \StrongWind) %\label{eq:align} 
\\
\DryConditions \;&:=\; \LowHumidity \,\lor\, (\HighTemperature \land \lnot \RainedRecently) \\
\Trigger \;&:=\; \LightningsFrequent \,\lor\, \lnot \LowHumanActivity \,\lor\, \PowerLinesNearby \\
\WildFireRisk \;&:=\; \Fuel \land \DryConditions \land \Trigger \label{eq:ground_truth} \tag{2}
\end{align}
\end{subequations}
The model architectures consists of (i) a generic CNN $h_\theta$ that processes each image to predict $\DenseForest$ and $\DryVegetation$ as latent concepts, and (ii) a logical component that combines these visual concepts with the environmental features to produce the final prediction.
Crucially, the CNN is not pretrained: it must learn to identify forests and dry vegetation solely via the gradients backpropagated through the logic.
We now analyze how LoH handles this task under different knowledge assumptions.} %See Appendix \ref{app:wildfire} for full implementation details.  

\color{black}

\paragraph{Full Knowledge (No Choice Operators).}
If the choice operator $[\cdot]$ is never used, the logical part has no learnable parameter, but the gradient can backpropagate from it to a neural part below. This corresponds to the NeSy approaches where knowledge is entirely specified a priori.  %In particular, this is how the Gödel Trick is used in \cite{Godel_trick}.

For the wildfire experiment, we set the LoH formula exactly to the ground-truth rule (\ref{eq:ground_truth}). %The only learnable component is the CNN $h_\theta$, which must learn to ground the visual predicates $DenseForest$ and $DryVegetation$ so that the symbolic layer can reproduce the labels.
Since the logic is fixed and correct, the learning focuses entirely on the perception module. The model successfully solves this symbol grounding problem, learning to map images to $\DenseForest$ and $\DryVegetation$ with high accuracy simply by minimizing the classification error of the $\WildFireRisk$ label (see Figure \ref{fig:wildfire}).

\paragraph{Selecting Reliable Rules.}
Suppose we have a knowledge set of $n$ candidate rules $r_1, \dots, r_n$, but we are unsure on whether they are correct. Then, we can use the following LoH formula to select a reliable subset:
\begin{equation}\label{eq:select_rules}
\bigwedge_{i=1}^{n} \; [r_i,\top] 
\end{equation}
Indeed, the hypothesis space spans all possible subsets: by selecting $r_i$ over $\top$, the model effectively picks such rule. When the rules are clauses (i.e., disjunctions of possibly negated propositions), this setup parallels \emph{KENN} \cite{daniele2019kenn}, which also have a learnable weight for each rule (even if it is never made discrete). %, but works with a continuous relaxation, unlike LoH. 
In our framework, each $r_i$ in \eqref{eq:select_rules} can be any formula. For example, we may have entire knowledge bases as $r_i$'s, and use \eqref{eq:select_rules} to decide which to trust. %Or, we may have some incomplete rules, whose completions can be learned with nested choose operators.

On the wildfire task, we build a pool of rules by taking the three ground-truth rules $\Fuel$, $\DryConditions$ and $\Trigger$ and augmenting them with $12$ plausible but
incorrect variants (see Appendix \ref{app:candidates}). Formula~\eqref{eq:select_rules} is then applied to this pool. The resulting learned subset is a data-driven refinement of the original possibilities.

\paragraph{Selecting One Rule per Set.}
Given $m$ sets $\{r_{i,1}, \dots, r_{i,n_i}\}$ of candidate rules, we may want to enforce the choice of \emph{exactly one} rule per set. This may happen for example because the rules in each set are mutually exclusive. %Or, as another example, because we want to learn a particular state of a system made of some categorical or continuous variables, and each of the $m$ sets of rules refers to a different variable. In both cases, 
For this purpose, we can use the following LoH formula:
\begin{equation}\label{eq:rules_fixed}
\bigwedge_{i=1}^{m} [r_{i,1},\,r_{i,2},\dots,r_{i,n_i}]
\end{equation}
For wildfire risk, suppose domain experts agree on the structure of the final rule as a conjunction of three components---fuel, dryness, trigger---but propose several alternatives for each component. We group these
into three sets and apply \eqref{eq:rules_fixed} with $m = 3$: one choice among fuel candidates, one among
dryness candidates, and one among trigger candidates (see Appendix \ref{app:candidates}). This regime differs from the previous one in that the model must commit to a single alternative within each group, instead of freely activating or deactivating rules.

\color{black}
\paragraph{Partial Knowledge Base.} 
If we have a knowledge base $K$ which is reliable but not complete, we can couple it with a rule-learning LoH formula $\Phi$, and use $K\land \Phi$. Similarly, we may have a knowledge base whose rules have some missing parts, and fill them with some rule-learning formulas.

As an example, in wildfire risk assessment, we consider the $\Fuel$ rule as given, while the rest is unknown and must be selected from the non-fuel candidate rules used in the previous settings (see Appendix~\ref{app:candidates}). %By freezing the valid $Fuel$ logic, the model effectively propagate the  the visual concepts $DenseForest$ and $DryVegetation$. And LoH learns which dryness- and trigger-related rules best complement the
For these candidates, we apply the subset-selection scheme of \eqref{eq:select_rules}, so that LoH learns which dryness- and trigger-related rules best complement the known $\Fuel$ rule. 

\redNote{
\paragraph{Results and Comparison to ILP.} Figure~\ref{fig:wildfire} compares the four knowledge regimes introduced so far: full knowledge, selecting reliable rules, selecting one rule per set, and partial knowledge base with $\Fuel$ given. As expected, richer priors yield better performance.

We also wanted to compare LoH against inductive logic programming (ILP). Among the differentiable ILP systems, we found no implementation that supports all regimes directly, by lacking a way to constrain the hypothesis spaces to operations such as ``select a subset of given rules''. We therefore added the required machinery starting from a $\partial$ILP \cite{evans2018learning} public implementation,\footnote{\href{ai-systems/DILP-Core}{https://github.com/ai-systems/DILP-Core}} and used such modified system as a baseline.\footnote{Concretely, we (i) added a \emph{mode-declaration}-style filter to restrict clause generation to prescribed sets (e.g., in ``selecting one rule per set'', $\Fuel/\DryVegetation/\Trigger$ can be defined only by one of their respective candidate rules); (ii) implemented the conjunction of multiple formulas via a conjunction structure with auxiliary predicates; (iii) grounded each candidate rule as a \emph{differentiable extensional predicate} using product fuzzy logic (since defining the rules intensionally would break the clause templates); and (iv) added mini-batch training, exploiting sample independence.} After the changes, $\partial$ILP assigns a learnable weight to each candidate rule through its clause-selection parameters. This effectively reproduce a choice mechanism under product fuzzy logic, with similar degrees of freedom as LoH, but in a more cumbersome and less direct way. 

Notably, in regimes where there is actually a logical part to be learned, LoH outperforms $\partial$ILP, with the performance gap widening in the settings involving the largest hypothesis spaces of logical rules (Selecting Reliable Rules, and Partial Knowledge Base).
}

%\setlength{\intextsep}{0pt}
%\begin{wrapfigure}{r}{0.5\textwidth}
\begin{figure}[ht]
    \centering
    \includegraphics[width=\linewidth]{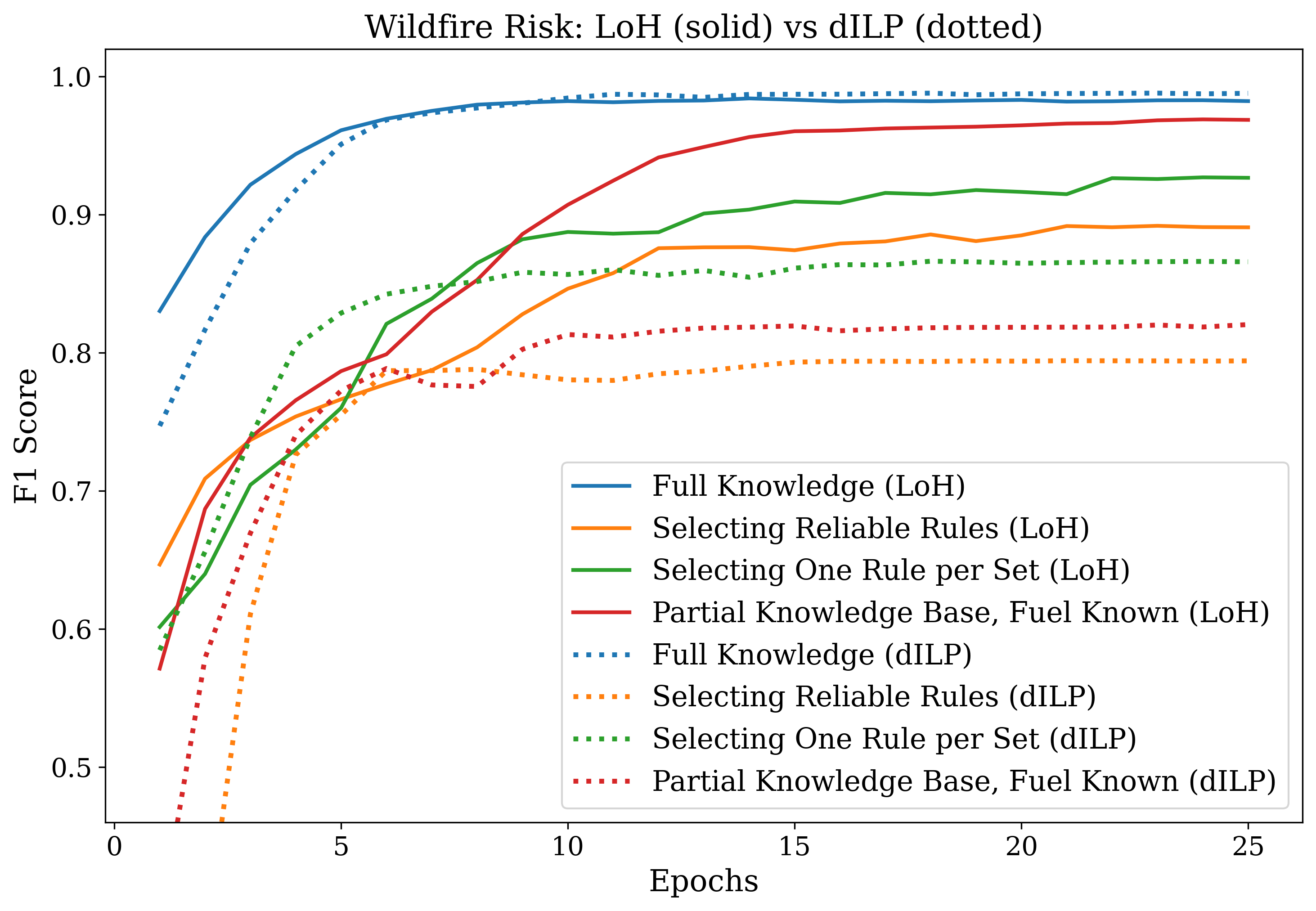}
    \caption{Average training curves, over 20 runs, of the LoH and $\partial$ILP models based on different levels of knowledge.} %, on the wildfire risk assessment task.}
    \label{fig:wildfire}
\end{figure}
%\end{wrapfigure}

\paragraph{Zero Knowledge (Pure Rule Learning).} For rule induction, we can combine neurons learning disjunctions and neurons learning conjunctions. A simple and most efficient way is to arrange them in layers and exploit parallel computation on tensors, like in standard Artificial Neural Networks. The following are LoH formulas for neurons learning the disjunction of a subset of neurons of the previous layer, with and without negations:
\begin{align}
    n^{(l+1)}_j &\coloneqq \bigvee_{i=1}^{m_l} [n^{(l)}_i, \neg n^{(l)}_i, \bot],  \\ %\quad
    n^{(l+1)}_j &\coloneqq \bigvee_{i=1}^{m_l} [n^{(l)}_i, \bot] \label{eq:neurons}
\end{align}
Analogously, conjunctive neurons simply replace disjunction with conjunction and $\bot$ with $\top$. The layers adopted in NLNs~\cite{payani2019inductivelogicprogrammingdifferentiable} and MLLPs~\cite{mllp} use neurons analogous to those, with product fuzzy logic instead of Gödel's. In this correspondence, they would use---as we suggested---conjunctive compilation for the choice operators in conjunctive neurons and vice versa for disjunctive neurons. However, LoH is flexible and many alternative designs of neurons are possible. For example, 
\begin{align}
    n^{(l+1)}_j & \coloneqq \bigvee_{i=1}^k [n^{(l)}_1,\dots, n^{(l)}_{m_l}, \neg n^{(l)}_1,\dots, \neg n^{(l)}_{m_l}],  \\
    %\quad  \text{and} \quad
    n^{(l+1)}_j & \coloneqq \bigvee_{i=1}^k [n^{(l)}_1,\dots, n^{(l)}_{m_l}] \label{eq:neurons2}
\end{align}
are neurons learning clauses of fixed width $k$ (possibly with repetitions)---$k$ being a hyperparameter.
A detailed empirical study of this zero-knowledge regime is deferred to Section~\ref{sec:experiments}.

\begin{comment}
{\color{red}{
Using the disjunctive compilations, the neurons defined in \eqref{eq:neurons} and \eqref{eq:neurons2} execute respectively the following operations:
\[
n_{j}^{(l+1)}
=
\max_{i=1,\dots,m_l}
\min(w_{ji},\, n_i^{(l)})
\]

\[
n_{j}^{(l+1)}
=
\max_{t=1,\dots,k}\,\max_{r=1,\dots,m_l}
\min(w_{jtr},\, n_r^{(l)})
\]
%where the weights $w_{ji}$ and $w_{jtr}$ are constructed in the way explained in Section \ref{sec:compilation}.
}}
\end{comment}

\paragraph{Respecting Syntactic Requirements.}
If the learned rules of a NeSy model are to be used also in a symbolic program, this may require them to adhere to a specific format or template. %In such a case, one can simply express the template in LoH. 
If it is possible to express the template in LoH, and LoH is quite flexible, then the adherence is guaranteed. As an example, here is a possible template for clauses of width~$3$: \begin{multline}\label{eq:template}
[v_1,\dots,v_n, \neg v_1,\dots,\neg v_n]
\;\lor\; \\
[v_1,\dots,v_n, \neg v_1,\dots,\neg v_n]
\;\lor\; \\
[v_1,\dots,v_n, \neg v_1,\dots,\neg v_n]
\end{multline}
And here is a template for definite clauses (i.e., clauses with exactly one positive literal):
\begin{equation}\label{eq:definite_clauses}
\bigvee_{i=1}^n [\neg v_i, \bot]
\;\lor\;
[v_1,\dots,v_n]
\end{equation}
%Appendix \ref{app:ablation} provides an experiment on enforcing these syntactic requirements, studying the convergence properties for different templates on a purely symbolic synthetic dataset. %We employ a distinct,  to better 
%For clarity, we investigate this ``syntactic requirements'' regime in Appendix~\ref{app:ablation} on a dedicated purely-propositional synthetic task where these templates can be inspected more clearly than in the wildfire setting.

%\vspace{2mm}
%\noindent
%Overall, these examples illustrate the remarkable flexibility of LoH. By mixing and matching different templates and choice operators, we can transition smoothly from purely known rules to completely data-driven rule discovery.
\section{Experiments}
\label{sec:experiments}

In this section, we evaluate the performance of our models, in their general form when no domain knowledge is provided: i.e., models made alternating conjunctive and disjunctive layers, with neurons of the form in \eqref{eq:neurons} or in \eqref{eq:neurons2}. 
We employ the suggested compilations with the Gödel trick, and test the performance on several benchmark classification datasets.  %All experiments were conducted on a cluster node equipped with an Nvidia RTX A5000 with 60GB RAM. 
%Code is available at \href{https://github.com/dbizzaro/LoH}{github.com/dbizzaro/LoH}.

\subsection{Classification on Tabular Datasets}\label{subsec:classification}

\paragraph{Datasets.} We use 12 classification benchmarks from the UCI Machine Learning Repository, available with CC-BY 4.0 license.
They were previously employed in \cite{mllp} for evaluating MLLP/CRS, a state-of-the-art rule-learning model using product fuzzy logic. 
As a preprocessing step, we adopt the same data discretization and binarization procedure as in \cite{mllp}: the recursive minimal entropy partitioning algorithm \cite{Dougherty1995}, followed by one-hot-encoding.  %Datasets' references and properties are available in Appendix \ref{app:datasets}.

\paragraph{Models.} 
In our model, we alternate conjunctive and disjunctive layers without negations. The choice between using neurons of type (\ref{eq:neurons}) or neurons of type (\ref{eq:neurons2}) is performed by the hyperparameters' tuning, together with the rest of the architecture (number and size of layers, and whether to start with a conjunctive or a disjunctive layer). 
We compare our model against Differentiable Logic Networks (DLN) \cite{petersen2022deep}, the aforementioned MLLP \cite{mllp}, and the rules extracted from it, which are called CRS. Note that for DLN we report the performance of the continuous model, which is generally higher than the discretized one. Instead, Gödel semantics ensures our model behaves identically before and after binarization, and CRS corresponds to the post-hoc binarization of MLLP. We also consider commonly used machine learning baselines: Decision Trees (DT), Random Forests (RF), XGBoost and standard fully-connected Neural Networks (NN).

Five of the twelve datasets have more than two classes, %\footnote{chess (18 classes), connect-4 (3 classes), letRecog (26 classes), nursery (5 classes) and wine (3 classes). A table summarizing the number of features and classes for each dataset is provided in Appendix G.} 
and we want to treat multi-class predictions as mutually exclusive output propositions. Therefore, in our model, we apply a reparameterization to the final layer, to guarantee that exactly one output per example exceeds the threshold value $0.5$---the output of the predicted class. This is analogous to the procedure used before for designing the weights of a choice operator, but does not require the addition of noise. Concretely, let $z_i$ be the logits of the outputs $o_i$ produced by the outmost layer---each one corresponding to a different class. Each logit gets shifted by subtracting the mean $\bar z$ of the two largest $z_i$'s. The resulting $z_i-\bar z$ values---of which, by construction, one and only one is positive---are then outputted through a sigmoid activation. %This enables gradient-based selection of the winning class without resorting to non-differentiable operators such as $argmax$.

\paragraph{Methodology.}
%\label{subsec:methdodology-classification}
Each dataset is divided into $20\%$ for testing and $80\%$ for training and validation. In particular, validation takes $12.5\%$ of the second split, so $10\%$ of the whole dataset. Since most datasets are unbalanced, we use the F$1$ score (macro) as classification metric. As loss functions, we use Binary Cross-Entropy for NN, DLN and our model, and Mean Squared Error for MLLP. All of them are optimized using  Adam \cite{adam}. 
For each dataset, the hyperparameters of each model are tuned using $80$ trials of the TPE algorithm \cite{NIPS2011_TPE} from the Optuna library. 
Architectural choices---such as the number and width of layers---are tuned alongside all other hyperparameters. %Appendix \ref{app:hyperparams} lists every tuned variable (and their search ranges), while the values ultimately selected are available in the code.
%Both MLLP and our models introduce a degree of randomness that can influence training outcomes significantly, even under identical settings. Hence, for such models, we perform two training runs for each non-pruned trial and average their metrics, to ensure a more robust hyperparameters' selection.

%The architectural options are tuned like the other hyperparameters, and the appendix lists all that is tuned, with ranges and values. 
%\ref{app:hyperparams}, 
%and the optimized values for each dataset are available with the code. 

\begin{comment}
\begin{table}
  \caption{Comparison of methods on the Visual Tic-Tac-Toe task.}
  \label{sample-table}
  \centering
  \begin{tabular}{lcccc}
    \toprule
    \multicolumn{2}{c}{Part}                   \\
    \cmidrule(r){1-2}
    Name     & Description     & Size ($\mu$m) \\
    \midrule
    Dendrite & Input terminal  & $\sim$100     \\
    Axon     & Output terminal & $\sim$10      \\
    Soma     & Cell body       & up to $10^6$  \\
    \bottomrule
  \end{tabular}
\end{table}
\end{comment}

%\cite{adult,bank_marketing,banknote,blogger,chess,connect_4,letRecog,magic,mushroom,nursery_76,tic-tac-toe,wine}
\begin{table*}[ht]
    \centering
    \begin{tabular}{lccccc|ccc}
    \toprule
        Dataset & DT & RF & XGBoost & NN &  MLLP & CRS & DLN & LoH %& Ensemble 
        \\ \midrule
        
        adult %\cite{adult} 
        & .80 \scriptsize{±.00} & .81 \scriptsize{±.00} & 
        \textbf{.82} \scriptsize{±.00} & .81 \scriptsize{±.01} 
        &
        .80 \scriptsize{±.01} & .75 \scriptsize{±.06} & 
        .80 \scriptsize{±.02} &
        \textbf{.81} \scriptsize{±.01} %& \textbf{.82 \scriptsize{±.01}} 
        \\ 
        
        bank-m. %\cite{bank_marketing} 
        & .74 \scriptsize{±.00} & .73 \scriptsize{±.00} & 
        .76 \scriptsize{±.00} & \textbf{.78} \scriptsize{±.01} & 
        .69 \scriptsize{±.08} & .75 \scriptsize{±.01} & 
        \textbf{.76} \scriptsize{±.01} &
        \textbf{.76} \scriptsize{±.01} %& .75 \scriptsize{±.01} 
        \\
        
        banknote %\cite{banknote} 
        & .95 \scriptsize{±.00} & .95 \scriptsize{±.01} & 
        \textbf{.96} \scriptsize{±.00} & \textbf{.96} \scriptsize{±.00} & 
        \textbf{.96} \scriptsize{±.00} & \textbf{.96} \scriptsize{±.00} & 
        .95 \scriptsize{±.01} &
        \textbf{.96} \scriptsize{±.00} %& \textbf{.96 \scriptsize{±.00}} 
        \\ 
        
        blogger %\cite{blogger} 
        & .64 \scriptsize{±.00} & .53 \scriptsize{±.00} 
        & .69 \scriptsize{±.00} & .82 \scriptsize{±.05} & 
        \textbf{.84} \scriptsize{±.00} & .78 \scriptsize{±.02} & 
        .72 \scriptsize{±.12} &
        \textbf{.83} \scriptsize{±.08} %& .82 \scriptsize{±.07} 
        \\
        
        chess %\cite{chess} 
        & .81 \scriptsize{±.00} & .58 \scriptsize{±.01} & 
        \textbf{.85} \scriptsize{±.00} & .82 \scriptsize{±.01} & 
        .83 \scriptsize{±.02} & \textbf{.74} \scriptsize{±.02} & 
        .41 \scriptsize{±.02} &
        .69 \scriptsize{±.01} %& .45 \scriptsize{±.03} 
        \\ 
        
        connect-4 %\cite{connect_4} 
        & .59 \scriptsize{±.00} & .55 \scriptsize{±.00} & 
        \textbf{.71} \scriptsize{±.00} & \textbf{.71} \scriptsize{±.04} &
        .58 \scriptsize{±.01} & .58 \scriptsize{±.01} & 
        \textbf{.62} \scriptsize{±.01} &
        .58 \scriptsize{±.01} %& .55 \scriptsize{±.01} 
        \\ 
        
        letRecog %\cite{letRecog} 
        & .80 \scriptsize{±.00} & .76 \scriptsize{±.00} & 
        .92 \scriptsize{±.00} & \textbf{.93} \scriptsize{±.00} & 
        .85 \scriptsize{±.01} & \textbf{.80} \scriptsize{±.01} & 
        .64 \scriptsize{±.02} &
        .77 \scriptsize{±.01} %& .71 \scriptsize{±.01} 
        \\ 
        
        magic04 %\cite{magic} 
        & .81 \scriptsize{±.00} & .83 \scriptsize{±.00} & 
        \textbf{.84} \scriptsize{±.00} & \textbf{.84} \scriptsize{±.00} & 
        \textbf{.84} \scriptsize{±.00} & .80 \scriptsize{±.00} & 
        \textbf{.83} \scriptsize{±.00} &
        \textbf{.83} \scriptsize{±.00} %& .83 \scriptsize{±.00} 
        \\ 
        
        mushroom %\cite{mushroom} 
        & \textbf{1.} \scriptsize{±.00} & \textbf{1.} \scriptsize{±.00} & 
        \textbf{1.} \scriptsize{±.00} & \textbf{1.} \scriptsize{±.00} & 
        \textbf{1.} \scriptsize{±.00} &\textbf{1.} \scriptsize{±.01} &
        \textbf{1.} \scriptsize{±.00} &
        \textbf{1.} \scriptsize{±.00} %& \textbf{1. \scriptsize{±.00}} 
        \\ 
        
        nursery %\cite{nursery_76} 
        & .79 \scriptsize{±.00} & .79 \scriptsize{±.00} & 
        .80 \scriptsize{±.00} & \textbf{1.} \scriptsize{±.00} & 
        \textbf{1.} \scriptsize{±.00} &
        \textbf{1.} \scriptsize{±.00} &  \textbf{1.} \scriptsize{±.00} & 
        \textbf{1.} \scriptsize{±.00} %& .97 \scriptsize{±.04} 
        \\
        
        tic-tac-toe %\cite{tic-tac-toe} 
        & .89 \scriptsize{±.00} & .98 \scriptsize{±.01} & 
        .97 \scriptsize{±.00} & .99 \scriptsize{±.01} & 
        \textbf{1.} \scriptsize{±.00} & \textbf{1.} \scriptsize{±.00} & 
        .99 \scriptsize{±.01} &
        .99 \scriptsize{±.01} %& .99 \scriptsize{±.01} 
        \\ 
        
        wine %\cite{wine} 
        & \textbf{1.} \scriptsize{±.00} & \textbf{1.} \scriptsize{±.01} & 
        .96 \scriptsize{±.00} & .97 \scriptsize{±.05} & 
        \textbf{1.} \scriptsize{±.00} & .96 \scriptsize{±.01} & 
        .97 \scriptsize{±.02} &
        \textbf{.98} \scriptsize{±.01} %& .98 \scriptsize{±.02} 
        \\ 
        \midrule

        mean (\textuparrow) & .82 & .79 & .86 & \textbf{.89} & .87 & .84 & .81 & \textbf{.85} %& .82 
        \\
        
        avg. rank (\textdownarrow) & 5.7 & 6.0 & 3.7 & \textbf{2.9} & 3.5 & 5.0 & 5.2 & \textbf{4.1} %& 4.2 
        \\
        \bottomrule
    \end{tabular}
    \caption{Mean F1 scores on the test sets of tabular benchmarks, out of $10$ runs. \redNote{Rightmost models are highlighted separately because they combine differentiable learning with symbolic rule extraction (DLN is reported in its soft form but can be discretized; CRS is the discretization of MLLP).}}
    \label{tab:uciml_results}
\end{table*}

\paragraph{Discussion.}
For the selected hyperparameters, Table~\ref{tab:uciml_results} provides the means and standard deviations of the test-set F$1$ scores, out of $10$ different training runs on the training plus validation sets.
Although vanilla neural networks attain the best summary score ($.89$) and rank ($2.9$), they---together with RF and XGBoost---do not allow to extract symbolic knowledge. MLLP---which come second---does support rule extraction, but with the loss in performance from it to CRS. In fact, only DT, CRS, the discretization of DLN, and our model achieve all the benefits of symbolic discrete rules, such as better interpretability, explainability, and efficiency of inference. 
Moreover, unlike DT, RF and XGBoost, our model is fully differentiable, so it could be placed downstream of perception modules (CNNs, transformers, etc.) and be trained end-to-end. 

Apart from the two datasets with the largest numbers of classes, namely \emph{chess} (with 18) and \emph{letRecog} (with 26), our model have consistently better scores than CRS, and is almost always on par with, or close to, the neural network baseline. This suggests that our handling of many-way classification may not be optimal. Notably, DLN struggles even more on those two benchmarks, and in general does not surpass our model even in its continuous (pre-discretization) form. Hence, when only a few classes are present, our model reliably ranks among the top performers while simultaneously providing symbolic rules and end-to-end differentiability.

\subsection{Visual Tic-Tac-Toe}
\label{subsec:visualttt}
One of the previous classification benchmarks \cite{tic-tac-toe} is based on the game of tic-tac-toe. The dataset provides all possible ending board configurations (for games in which $X$ begins), and the target is to predict whether $X$ has won the game or not. So the target function can be expressed with a simple formula in Disjunctive Normal Form (DNF), having $8$ clauses.\footnote{Let the input data be encoded with propositions $X_1$-$X_9$, $O_1$-$O_9$ and $B_1$-$B_9$, meaning that $X_i$ is true if there is an $X$ at position $(\lfloor i/3\rfloor, i\%3)$, and similarly for $O$s and $B$lank cells. Then, the target function is ${(X_1 \land X_2 \land X_3)} \lor {(X_4 \land X_5 \land X_6)} \lor {(X_7 \land X_8 \land X_9)} \lor {(X_1 \land X_4 \land X_7)} \lor {(X_2 \land X_5 \land X_8)} \lor  {(X_3 \land X_6 \land X_9)} \lor {(X_1 \land X_5 \land X_9) \lor (X_3 \land X_5 \land X_7)}$.}
%if the input data is encoded with propositions $X_1$-$X_9$, $O_1$-$O_9$ and $B_1$-$B_9$,\footnote{Meaning that $X_i$ is true if there is an $X$ at position $(\lfloor i/3\rfloor, i\%3)$, and similarly for $O$s and $B$lank cells.} the target function can be expressed with the following simple  formula in Disjunctive Normal Form (DNF): 
%\begin{multline*}\label{eq:ttt}
%(X_1 \land X_2 \land X_3) \lor (X_4 \land X_5 \land X_6) \lor (X_7 \land X_8 \land X_9) \\ \lor(X_1 \land X_4 \land X_7) \lor (X_2 \land X_5 \land X_8) \lor  (X_3 \land X_6 \land X_9)  \\ \lor (X_1 \land X_5 \land X_9) \lor (X_3 \land X_5 \land X_7)
%\end{multline*}
Let us introduce a more challenging variant of the dataset, which we refer to as \emph{Visual Tic-Tac-Toe}. In this version, instead of symbolic board encodings, each cell is represented by a MNIST image. Specifically, we assign images of the digit $0$ to represent $X$, images of the digit $1$ to represent $O$, and blank cells are represented by images of the digit $2$. As a result, instead of a structured propositional encoding, the inputs consist of $3 \times 3$ grids of grayscale images.
This modification significantly increases the difficulty of the task, as models must now also learn to recognize digit representations from the images, but with as little supervision as before (just the ``$X$ winning'' label). Traditional symbolic models would struggle with such high-dimensional and noisy input, making this an interesting benchmark for NeSy learning.

\paragraph{Models.} To handle the image-based input, we extend all models with a standard CNN. This network consists of two convolutional layers, each using a kernel size of $3$ with padding $1$, followed by ReLU activations and max-pooling. The first convolutional layer has $32$ output channels, while the second has $64$. After feature extraction, the CNN flattens the output and passes it through a fully connected layer of size $128$, a dropout layer with probability $0.5$, and another fully connected layer of size $3$. The three outputs for each of the nine images composing a grid are then concatenated. 

For DLN, MLLP/CRS and our models, the outputs of the CNN for the nine images are to be considered as input ``propositions'', so their values are clipped between $0$ and $1$. The learning process is end-to-end, including the CNN component. Crucially, the CNN is not pretrained to recognize the digits, nor are its outputs predefined to represent anything specific. %Instead, the CNN is trained together with the rule-learning logical network, in an end-to-end way. 
The hyperparameters (and their tuning) are as before, but we allow the CNN part to have a separate learning rate (in the range $10^{-5}$--$10^{-3}$). %from the one of our model (in the range $0.01$-$0.2$). %\footnote{Indeed, our model works with larger-than-usual learning rates (but with sparse gradients).}
Moreover, for both MLLP/CRS and our model, we fix the number of layers to $2$ and distinguish the two cases of whether the last one is disjunctive (DNF) or conjunctive (CNF). This is not possible for DLN.

We also implemented a neural baseline that follows a classical deep learning approach (NN). This model uses a CNN module as the one described earlier, but its outputs are not clipped, and their number is a hyperparameter in the range $3$-$32$ (instead of being fixed to $3$). This is because in this case the outputs of this part of the network may be better seen as embeddings rather than symbols. The concatenation of such vectors (for the nine images in a grid) is then fed to a multi-layer perceptron.

\paragraph{Methodology.}
%\label{subsec:methdodology-visualttt}
The original symbolic tic-tac-toe dataset is split with the same proportions as before ($70\%$-$10\%$-$20\%$). %($70\%$ and $10\%$ for training and validation, and $20\%$ for testing). 
The training and validation parts are given images from MNIST training set, while the test part is given images from MNIST test set. No image is used more than once. Moreover, for each board configuration in the training and validation sets, two different image grids are created, effectively doubling the number of samples. 

%Table~\ref{tab:visualttt_results} provides the mean and standard deviations of the test-set F$1$ scores, based on training each model $30$ times on the combined training and validation sets. What changes between the runs is the network initialization, the mini-batches, the added noise in our model and the random binarization occurring in MLLP as a form of regularization.

Beyond performance, a key advantage of NeSy approaches lies in their better interpretability. To interpret the learned decision rules, we first want to assign meaningful labels to the propositions (that corresponds to the CNN outputs units, for each position in the $3\times3$ grid). %These semantic labels are obtained post-hoc: 
After the training, we therefore inspect the three output units of the CNN, and determine whether a unit consistently responds to MNIST digits representing $X$, $O$, or blank cells. We assign $X$, $O$, or $B$ symbols based on such firing behavior, and use such names to extract and evaluate the learned logical formulas.
%After this alignment step, we interpret the corresponding units as representing $X$, $O$, and $B$ a symbolic evaluation of the extracted formula.

%The exact procedure for such labeling is reported in Appendix \ref{app:visual_ttt}. %, together with an example of the decision rules learned by each model. 
%Table~\ref{tab:visualttt_results} also provides the average F$1$ scores of these extracted rules on the purely symbolic tic-tac-toe dataset.

\begin{comment}
\begin{table}[ht]
    \centering
    \begin{tabular}{ll}
    \hline
        Methods & F$1$ score & F$1$ score extracted formula \\ \hline
        CNN & .910 \scriptsize{±.138} & - \\
        MLLP-CNF & .922 \scriptsize{±.022} & - \\
        MLLP (DNF) & .922 \scriptsize{±.022} & - \\
        CRS (CNF) & .912 \scriptsize{±.027} & \\
        CRS (DNF) & .912 \scriptsize{±.027} & \\
        Ours  & .944 \scriptsize{±.030} \\
        Ours  & \textbf{.951 \scriptsize{±.014}} \\
         \hline
    \end{tabular}
    \caption{Comparison of methods on the Visual Tic-Tac-Toe task.}
    \label{tab:visualttt_results}
\end{table}
\end{comment}

\begin{table*}[ht]
  \centering
  \begin{tabular}{lcc*{3}{cc}}
    \toprule
    \multicolumn{3}{l}{} &
    \multicolumn{2}{c}{MLLP} &
    \multicolumn{2}{c}{CRS} &
    \multicolumn{2}{c}{LoH}\\
    \cmidrule(lr){4-5}\cmidrule(lr){6-7}\cmidrule(lr){8-9}
     & NN & DLN & DNF & CNF & DNF & CNF & DNF & CNF \\
    \midrule
    NeSy eval & .91 \scriptsize{±.14} & .96 \scriptsize{±.01}  & .80 \scriptsize{±.22} & .94 \scriptsize{±.02} & .76 \scriptsize{±.28} & .92 \scriptsize{±.00} & \textbf{.97 \scriptsize{±.01}} & .95 \scriptsize{±.01} \\
    Symbolic eval  & -  & .37 \scriptsize{±.21} & - & - & .80 \scriptsize{±.30} & .97 \scriptsize{±.02} & \textbf{.99 \scriptsize{±.00}} & \textbf{.99 \scriptsize{±.00}} \\
    \bottomrule
  \end{tabular}
  \caption{Comparison of the models on the Visual Tic-Tac-Toe task.} \label{tab:visualttt_results}
\end{table*}

\paragraph{Discussion.}
Table~\ref{tab:visualttt_results} provides the mean and standard deviations of the test-set F$1$ scores, based on training each model $30$ times on the combined training and validation sets. What changes between the runs is the network initialization, the mini-batches, the added noise in our model and the random binarization occurring in MLLP as a form of regularization. When possible, we report also the performance of the extracted formulas (Symbolic eval).

%MLLP in the DNF setting fails to learn any meaningful mapping, with its binarization CRS being a formula always True.
On $2$ out of its $30$ training runs, the neural baseline remained stuck at always predicting the most frequent class, with macro F$1$ score of $.39$. Instead, on the other runs, its performance was comparable to that of our CNF model. Similarly, MLLP/CRS in the DNF setting remained stuck on $4$ runs, with the remaining $26$ performing slightly worse than our CNF model.  
Finally, MLLP in the CNF version, and DLN, have similar performance to our CNF model on all runs. However, the symbolic evaluation reveals that DLN fails when discretized, and MLLP/CRS learned less accurate symbolic rules than ours. 
Moreover, the DNF version of our model is consistently better than all other models, and also found the 100\%-correct formula reported above on $4$ occasions. 
These results highlight the ability of our models to recover symbolic decision rules with high fidelity, even when the symbols must be learned from continuous high-dimensional perceptions.

%On $28$ out of the $30$ training runs, the performance of the NN baseline was comparable to that of our CNF model. However, while the got always results above $.80$, the two other runs of the CNN baseline remained stuck at always predicting the most frequent class, with macro F$1$ score of $.39$. Thus, both our base and ensemble model outperformed the CNN baseline.  %This remains true (with $p=0.002$ for the base model) also with respect to MLLP, the non-binarized (and therefore less interpretable) version of CRS.

%Furthermore, our models also surpassed CRS, with a statistical significance of $p < 0.001$.

%When testing the learned logical formulas against the original tic-tac-toe symbolic dataset, the ones obtained from our base model achieved $.981$~{\scriptsize $\pm .024$} mean F$1$ score, outperforming those obtained from CRS, which achieved $.969$~{\scriptsize $\pm .013$}. This result highlights the ability of our model to recover symbolic decision rules with high fidelity, even when the symbols must be learned from continuous high-dimensional perceptions.

\section{Conclusion and Future Work}
\label{sec:future_work}

We introduced a single, compact language for expressing hypothesis spaces of logical formulas and compiling them into differentiable models whose discrete rule extraction is provably loss-free under Gödel semantics. LoH unifies knowledge injection and rule induction within one propositional paradigm, and yields strong results on both tabular and perceptual benchmarks, while retaining the possibility to extract learned logical formulas following arbitrary templates. 

Despite these encouraging results, the present work leaves two important avenues open. First, the empirical validation should be broadened, targeting larger datasets, use of partial knowledge and additional NeSy tasks with complex perceptions. 
%Second, the formalism is so far propositional. Extending LoH with first-order logic quantifiers would unlock relational reasoning and allow direct comparison with first-order NeSy learners. Addressing these two limitations with richer logic and wider experimentation constitutes our next research milestone.
Second, LoH is currently propositional, while many NeSy applications are relational. In practice, many first-order systems already ground rules over a finite domain and then operate propositionally; the genuinely first-order aspect is the quantifier-driven aggregation over groundings. A first-order LoH could internalize this by adding $\forall$ and $\exists$ to the language, interpreting them as conjunctions/disjunctions over substitutions, and allowing choice operators to range over quantified and quantifier-free formulas alike. The main challenge would remain scalability due to the combinatorial blow-up of groundings, as in many FOL-based NeSy approaches.

%\section*{Reproducibility Statement}
%\label{sec:reproducibility} Assumptions, design choices and claims are reported in the main text and further explained in Appendices \ref{app:godeltrick} and \ref{app:choice_variables}. Theorem \ref{th:binarization} is a generally known result, and a proof can be found in \cite{Godel_trick}. We share the code in the supplementary materials, and will make it public upon acceptance. Experimental setups---including dataset preprocessing, splits, metrics, protocol, etc.---are described in Section \ref{sec:experiments}. Appendix~\ref{app:hyperparams} reports complete hyperparameter ranges, and the values taken for each benchmark are available within the code repository. Appendix~\ref{app:datasets} references the tabular datasets, and the code provide a way to build the Visual Tic-Tac-Toe benchmark. Finally, Appendix~\ref{app:visual_ttt} explains the symbol-labeling procedure utilized in the Visual Tic-Tac-Toe experiment.

\appendix
\appendix
\section{LoH with Placeholders: a Formalization}\label{app:choice_variables}

%\paragraph{Placeholder Declarations.}
We allow the user to \emph{name} any LoH subformula by writing a declaration of the form $p \coloneqq \phi$, 
where $\phi$ is a LoH formula and $p$ is a fresh identifier (i.e., a name different from any propositional variable and any other identifier). A placeholder can itself mention other placeholders that have been declared earlier (the directed graph of such references must be acyclic, to avoid circular definitions).\footnote{Notice that one can still build LoH-based recurrent networks exactly as one builds ordinary RNNs: by passing the hidden state from one time step to the next. In this case, one may write an LoH formula formally relating some placeholders to the ones at the previous time step, and acyclicity is in the unrolled architecture.} 
Given a well-defined %and topologically ordered 
list of declarations $\mathcal{D}=\{p_1 {\coloneqq} \phi_1, \dots , p_m {\coloneqq} \phi_m\}$, 
Algorithm~\ref{alg:hypotheses} produces the hypothesis space $\mathcal{H}(\Phi)$ for every LoH formula~$\Phi$ possibly containing the placeholders $p_1,\dots,p_m$.

\newcommand{\Env}{\mathcal{I}}
\newcommand{\EnvSpace}{\mathcal{S}}
\newcommand{\Eval}{\textsc{E}}
\newcommand{\Hspace}{\mathcal{H}}

\begin{algorithm}[ht]
\caption{Construction of the hypothesis space $\Hspace(\Phi)$}\label{alg:hypotheses}
\renewcommand{\algorithmicrequire}{\textbf{Input:}}
\renewcommand{\algorithmicensure}{\textbf{Output:}}
\begin{algorithmic}[1]
\Require LoH formula $\Phi$; declaration list $\mathcal{D}=\{p_i \coloneqq \phi_i\}_{i=1}^m$
\Ensure  $\Hspace(\Phi)$, the hypothesis space of $\Phi$
\Statex
\Function{Hypotheses}{$\Phi,\mathcal{D}$}
    \State \textbf{Tagging.}  Traverse every choice node $[\Psi_1,\dots,\Psi_n]$ in both $\Phi$ and the bodies $\phi_i$;
           assign a fresh identifier $c$ and arity $n_c$ to each.
    \State \textbf{Indices Space.}  $\EnvSpace \gets \prod_{c}\{1,\dots ,n_c\}\,$ (Cartesian pr.)
    %\State \textbf{Placeholder functions.}  For each declaration $p_i\coloneqq\phi_i$ define
    %       $\Phi_i(\Env) \coloneqq\ \Eval(\phi_i,\Env)$ (see next step for $\Eval$)
    \State \textbf{Evaluator.}
          Define the function $\Eval(\cdot,\Env)$ recursively:
          \begin{align}
             & \Eval(v,\Env)             \coloneqq v \tag{\text{R1'}}\\
             & \Eval(\lnot \Psi_1,\Env)     \coloneqq \lnot\Eval(\Psi_1,\Env) \tag{\text{R2'}}\\
             & \Eval(\Psi_1\land \Psi_2,\Env)  \coloneqq \Eval(\Psi_1,\Env)\land\Eval(\Psi_2,\Env) \tag{\text{R3'}}\\
             & \Eval(\Psi_1\lor \Psi_2,\Env)   \coloneqq \Eval(\Psi_1,\Env)\lor\Eval(\Psi_2,\Env) \tag{\text{R4'}}\\
             & \Eval([\Psi_1,\dots,\Psi_n]_c,\Env) \coloneqq \Eval(\Psi_{\Env[c]},\Env) \tag{\text{R5'}}\\
             & \Eval(p_i,\Env)           \coloneqq \Eval(\phi_i,\Env) \tag{\text{Placeholders}}
          \end{align}
    \State \textbf{Enumeration.}  $\Hspace(\Phi) \;\gets\; \bigl\{\Eval(\Phi,\Env)\;\bigl|\; \Env\in\EnvSpace \bigr\}$
    \State \Return $\Hspace(\Phi)$
\EndFunction
\end{algorithmic}
\end{algorithm}
%\vspace{0.5em}

%\paragraph{Correctness proof (sketch).}
%Each syntactic choice node of the LoH formula $\Phi$ is tagged with a unique gate identifier. Fixing an index for every gate yields an \emph{environment} $e\in\EnvSpace$, and the Cartesian product of all local domains is the full set of possible environments. The recursive evaluator $\textsc{Eval}(\cdot,e)$ mirrors the forward pass of the compiled decision‑gated model node by node, so for every environment we have $\textsc{Eval}(\Phi,e)=\text{output}(\Phi,e)$. Because the mapping $e\mapsto$ gate configuration is bijective, the algorithm’s enumeration $\{\textsc{Eval}(\Phi,e)\mid e\in\EnvSpace\}$ is (i) \emph{sound}, as every element can be obtained from the model, and (ii) \emph{complete}, since every possible gate configuration is realized by some $e$. Hence the returned set coincides exactly with the hypothesis space, $\mathcal{H}(\Phi)=\{\textsc{Eval}(\Phi,e)\mid e\in\EnvSpace\}$. 

Given $\Phi$ and $\mathcal{D}$, the Cartesian product $\EnvSpace$ (in line 3) is exactly the set of all possible assignments of concrete choices to every choice operator. Indeed, for every $\Env\in\EnvSpace$, the algorithm considers an index $\Env[c] \in \{1,\dots, n_c\}$ for every choice node $c$ (of arity $n_c$), %Consequently, the Cartesian product $\EnvSpace$ is exactly the set of all possible assignments of concrete indices to those parameters.
and the evaluator $\Eval(\cdot,\Env)$ is a function that (i) respects placeholder sharing and (ii) replaces every choice node $[\Psi_1,\dots,\Psi_{n_c}]_c$ with the branch $\Psi_{\Env[c]}$. Thus, $\Eval(\Phi,\Env)$ reflects the discrete model obtained compiling $\Phi$ (as in Section \ref{sec:compilation}), when $\Env[c]$ is the index of the unique weight greater than $0.5$, for each compiled choice operator $c$.
%Hence, $\Eval(\Phi,\Env)$ yields a unique fully specified propositional formula $\Psi_e$.  Distinct environments differ at least at one choice identifier, so the map $\Env\mapsto \Psi_e$ is injective; conversely, any fully discretized model must select a branch for every $c$, and therefore corresponds to some $\Env$, establishing surjectivity.  
It follows that $\mathcal{H}(\Phi)$ coincides exactly with the hypothesis space of the compiled and discretized model, as we wanted.

\redNote{When compiled, placeholders implement \emph{weight sharing}: all occurrences of the same placeholder reuse the same choice parameters, so a single shared selection is made wherever the placeholder is referenced.}

\section{Wildfire Risk: Candidate Rules}
\label{app:candidates}
\redNote{
Section~6 compares several ways of exploiting symbolic knowledge. The task is defined over nine propositional variables. Two of them are \emph{visual} concepts, to be formed by a CNN from the images: $\DenseForest$ and $\DryVegetation$ (see Figure \ref{fig:wildfire_images}). The remaining seven are \emph{non-visual} features provided in tabular form: $\LowHumidity$, $\HighTemperature$, $\RainedRecently$, $\StrongWind$, $\LightningsFrequent$, $\LowHumanActivity$ and $\PowerLinesNearby$.
The ground-truth label $\WildFireRisk$ is obtained by combining three interpretable intermediate rules: see Equations \eqref{eq:wildfire}. 
%The starting point is a small pool of candidate rules (one ground-truth rule plus some distractors).

Across the knowledge regimes, we start from three sets of candidate rules (for Fuel, Dryness, and Trigger), each comprising one ground-truth rule and a set of distractors.
We index candidates as $\phi_{i,j}$, where $i \in \{\Fuel,\DryConditions,\Trigger\}$ and
$j \in \{1,\dots,n_i\}$. The ground truth rules are those with $j=1$.

\begin{itemize}
  \item \textbf{Fuel candidates} ($n_{\Fuel}=4$):
  \begin{align*}
    \phi_{\Fuel,1} &:= \DenseForest \lor (\DryVegetation \land \StrongWind) \\
    \phi_{\Fuel,2} &:= \DenseForest \\
    \phi_{\Fuel,3} &:= \DenseForest \land \DryVegetation \\
    \phi_{\Fuel,4} &:= \DryVegetation \land (\StrongWind \lor \LowHumidity)
  \end{align*}

  \item \textbf{Dryness candidates} ($n_{\DryConditions}=6$):
  \begin{align*}
    \phi_{\DryConditions,1} &:= \LowHumidity \lor (\HighTemperature \land \neg \RainedRecently) \\
    \phi_{\DryConditions,2} &:= \LowHumidity \land \HighTemperature \land \neg \RainedRecently \\
    \phi_{\DryConditions,3} &:= \HighTemperature \lor (\LowHumidity \land \neg \RainedRecently) \\
    \phi_{\DryConditions,4} &:= \neg \RainedRecently \lor (\LowHumidity \land \StrongWind) \\
    \phi_{\DryConditions,5} &:= \neg \RainedRecently \lor (\LowHumidity \land \HighTemperature) \\
    \phi_{\DryConditions,6} &:= \DryVegetation \land \neg \RainedRecently
  \end{align*}

  \item \textbf{Trigger candidates} ($n_{\Trigger}=5$):
  \begin{align*}
    \phi_{\Trigger,1} &:= \LightningsFrequent \lor \neg \LowHumanActivity \lor \PowerLinesNearby \\
    \phi_{\Trigger,2} &:= \LightningsFrequent \land \LowHumanActivity \\
    \phi_{\Trigger,3} &:= \PowerLinesNearby \land \StrongWind \\
    \phi_{\Trigger,4} &:= \RainedRecently \land \LightningsFrequent \\
    \phi_{\Trigger,5} &:= \neg \LowHumanActivity
  \end{align*}
\end{itemize}

Let $\mathcal{I} := \{\Fuel,\DryConditions,\Trigger\}$. In the following we report the LoH formula used for each knowledge regime. 

\begin{itemize}
  \item \textbf{Full Knowledge:}
  \[
    \phi_{\Fuel,1} \;\land\; \phi_{\DryConditions,1} \;\land\; \phi_{\Trigger,1}
  \]

  \item \textbf{Selecting Reliable Rules:}
  \[
    \bigwedge_{i \in \mathcal{I}}\;\bigwedge_{j=1}^{n_i}\;[\,\phi_{i,j},\;\top\,]
  \]
  %(Choosing $\phi_{i,j}$ activates that rule; choosing $\top$ drops it.)

  \item \textbf{Selecting One Rule per Set:}
  \[
    \bigwedge_{i \in \mathcal{I}}\;[\,\phi_{i,1},\,\phi_{i,2},\,\dots,\,\phi_{i,n_i}]
  \]

  \item \textbf{Partial Knowledge Base, Fuel Known:}
  \[
    \phi_{\Fuel,1}
    \;\land\;
    \bigwedge_{i\in \mathcal{I}\setminus\{\Fuel\}}\bigwedge_{j=1}^{n_{i}}[\,\phi_{i,j},\top\,]
  \]
\end{itemize}
}

%\paragraph{Note on candidate-set subdivision.}
%We separate candidates into Fuel/Dryness/Trigger sets only for the
%\emph{Selecting One Rule per Set} regime. In the \emph{Selecting Reliable Rules} and
%\emph{Partial Knowledge Base} regimes, all formulas are treated as members of a single
%undivided pool (the indices $i,j$ above still uniquely identify them).

\begin{figure}[t]
  \centering
  % Dense forest, NOT dry vegetation
  \begin{subfigure}[b]{0.11\textwidth}
    \centering
    \includegraphics[width=\linewidth]{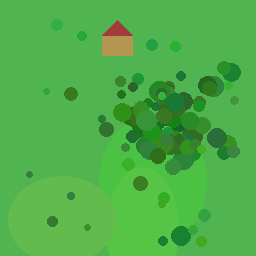}
    \caption{\small $F, \neg DV$}
  \end{subfigure}
  % Dense forest, dry vegetation
  \begin{subfigure}[b]{0.11\textwidth}
    \centering
    \includegraphics[width=\linewidth]{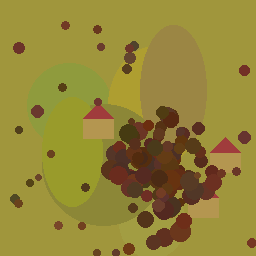}
    \caption{\small $F, DV$}
  \end{subfigure} %\hfill
  % No dense forest, NOT dry vegetation
  \begin{subfigure}[b]{0.11\textwidth}
    \centering
    \includegraphics[width=\linewidth]{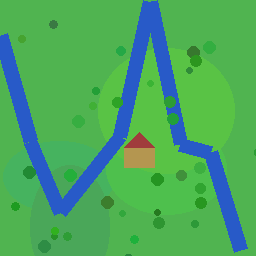}
    \caption{\small $\neg F, \neg DV$}
  \end{subfigure} 
  % No dense forest, dry vegetation
  \begin{subfigure}[b]{0.11\textwidth}
    \centering
    \includegraphics[width=\linewidth]{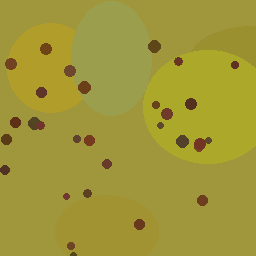}
    \caption{\small $\neg F, DV$}
  \end{subfigure}

  \caption{Examples of synthetic wildfire images for the four possible combinations of $\DenseForest$ ($F$) and $\DryVegetation$ ($DV$).} \label{fig:wildfire_images}
\end{figure}

\section{Formulas Length}
\redNote{
To evaluate the interpretability in terms of syntactic compactness, we compared the length of the propositional formulas extracted from CRS and LoH. We restricted this analysis to the binary classification datasets of Section \ref{subsec:classification}. To ensure a fair comparison, we constrained the architectures to yield a formula in disjunctive normal form (DNF) by design: a single hidden conjunctive layer encoding candidate clauses, followed by a disjunctive output neuron. %DNF formulas are easy to interpret because they are lists of sufficient conditions for the positive class. 
Except for this architectural constraint, all other hyperparameters and protocol details were tuned exactly as in the main experiment. Table~\ref{tab:length} reports the number of clauses and the mean clause width, averaged over $10$ independent runs. Before computing such values, we applied a simple pruning procedure following \cite{mllp}: we removed clauses that subsume others and are therefore redundant, and we discarded clauses that were never satisfied by any instance in the training set.

\begin{table}[ht]
%\color{red}
    \centering
    \caption{Average number of clauses and average clause size of the extracted DNF formulas, over 10 runs.}
    \label{tab:length}
    \begin{tabular}{l*{2}{cc}}
    \toprule
    \multicolumn{1}{l}{Dataset} &
    \multicolumn{2}{c}{CRS} &
    \multicolumn{2}{c}{LoH}\\
    \cmidrule(lr){2-3}\cmidrule(lr){4-5}
    & \#Clauses & Cl. Size & \#Clauses & Cl.\ Size  \\
     \midrule
        
        adult & 10.8 & 2.4 & 17.3 & 3.4  
        \\ 
        
        bank-m. &  118.8 & 6.1 & 14.3 & 4.8 
        \\
        
        banknote & 8.3 & 1.9 & 10.5 & 2.5 
        \\ 
        
        blogger & 6.6 & 2.8 & 9.1 & 3.1 
        \\
        
        magic04 & 73.7 & 2.7  &  30.6 & 2.8 
        \\ 
        
        mushroom & 7.7  & 3.9 & 10.0 & 7.6 
        \\ 
        
        tic-tac-toe & 8.0 & 3.0 & 12.2 & 3.5
        \\ 
        
        \bottomrule
    \end{tabular}
\end{table}

Overall, we find that CRS tends to produce smaller formulas than LoH in many simple cases. However, for \emph{bank-marketing} \cite{bank_marketing} and \emph{magic04} \cite{magic}, the extracted CRS formulas are extremely long outliers. In contrast, LoH does not exhibit such explosions in formula length.
}

\section*{AI Declaration}
We used LLMs as general-purpose assist tools for polishing prose, for minor code completions, and for generating the code used to produce the synthetic wildfire images. All such outputs were reviewed and verified by the authors, who take full responsibility.

\ifarxiv
\section*{Author Contributions}
\textbf{A.D.} conceived the idea, developed an initial prototype, supervised the project, and proofread the manuscript.  

\noindent \textbf{D.B.} carried out the software implementation, conducted all the experiments, and wrote the manuscript.
\fi

\bibliographystyle{kr}
\bibliography{references}

@article{Dougherty1995,
author = {Dougherty, James and Kohavi, Ron and Sahami, Mehran},
year = {1997},
month = {09},
pages = {},
title = {Supervised and Unsupervised Discretization of Continuous Features},
volume = {1995},
isbn = {9781558603776},
journal = {ICML},
doi = {10.1016/B978-1-55860-377-6.50032-3}
}

@misc{tic-tac-toe,
  author       = {Aha, David},
  title        = {{Tic-Tac-Toe Endgame}},
  year         = {1991},
  howpublished = {UCI Machine Learning Repository},
  doi         = {10.24432/C5688J},
}

@inproceedings{NIPS2011_TPE,
 author = {Bergstra, James and Bardenet, R\'{e}mi and Bengio, Yoshua and K\'{e}gl, Bal\'{a}zs},
 booktitle = {Advances in Neural Information Processing Systems},
 publisher = {Curran Associates, Inc.},
 title = {Algorithms for Hyper-Parameter Optimization},
 url = {https://proceedings.neurips.cc/paper_files/paper/2011/file/86e8f7ab32cfd12577bc2619bc635690-Paper.pdf},
 volume = {24},
 year = {2011}
}

@article{badreddine2022logic,
  title={Logic tensor networks},
  author={Badreddine, Samy and Garcez, Artur d'Avila and Serafini, Luciano and Spranger, Michael},
  journal={Artificial Intelligence},
  volume={303},
  pages={103649},
  year={2022},
  publisher={Elsevier}
}

@incollection{besold2021neural,
  title={Neural-symbolic learning and reasoning: A survey and interpretation},
  author={Besold, Tarek R and Bader, Sebastian and Bowman, Howard and Domingos, Pedro and Hitzler, Pascal and K{\"u}hnberger, Kai-Uwe and Lamb, Luis C and Lima, Priscila Machado Vieira and de Penning, Leo and Pinkas, Gadi and others},
  booktitle={Neuro-Symbolic Artificial Intelligence: The State of the Art},
  pages={1--51},
  year={2021},
  publisher={IOS press}
}

@article{campero2018logical,
  title={Logical rule induction and theory learning using neural theorem proving},
  author={Campero, Andres and Pareja, Aldo and Klinger, Tim and Tenenbaum, Josh and Riedel, Sebastian},
  journal={arXiv preprint arXiv:1809.02193},
  year={2018}
}

@inproceedings{daniele2019kenn,
  author    = {Daniele, Alessandro and Serafini, Luciano},
  title     = {Knowledge Enhanced Neural Networks},
  booktitle = {Proceedings of the 16th Pacific Rim International Conference on Artificial Intelligence ({PRICAI})},
  series    = {Lecture Notes in Computer Science},
  pages     = {542--554},
  publisher = {Springer},
  year      = {2019}
}

@article{diligenti2017semantic,
  title={Semantic-based regularization for learning and inference},
  author={Diligenti, Michelangelo and Gori, Marco and Sacca, Claudio},
  journal={Artificial Intelligence},
  volume={244},
  pages={143--165},
  year={2017},
  publisher={Elsevier}
}

@article{evans2018learning,
  title={Learning explanatory rules from noisy data},
  author={Evans, Richard and Grefenstette, Edward},
  journal={Journal of Artificial Intelligence Research},
  volume={61},
  pages={1--64},
  year={2018}
}

@article{manhaeve2018deepproblog,
  title={Deepproblog: Neural probabilistic logic programming},
  author={Manhaeve, Robin and Dumancic, Sebastijan and Kimmig, Angelika and Demeester, Thomas and De Raedt, Luc},
  journal={Advances in neural information processing systems},
  volume={31},
  year={2018}
}

@article{Godel_trick,
      title={Gradient-Based Optimization on G\"odel Logic as Discrete Local Search}, 
      author={Alessandro Daniele and Emile van Krieken},
      year={2026},
      journal={arXiv preprint arXiv:2503.01817},
}

@article{payani2019inductivelogicprogrammingdifferentiable,
      title={Inductive Logic Programming via Differentiable Deep Neural Logic Networks}, 
      author={Ali Payani and Faramarz Fekri},
      year={2019},
      journal={arXiv preprint arXiv:1906.03523},
}

@article{MARRA2024104062,
title = {From statistical relational to neurosymbolic artificial intelligence: A survey},
journal = {Artificial Intelligence},
volume = {328},
pages = {104062},
year = {2024},
issn = {0004-3702},
doi = {https://doi.org/10.1016/j.artint.2023.104062},
author = {Giuseppe Marra and Sebastijan Dumančić and Robin Manhaeve and Luc {De Raedt}},
}

@article{daniele2022deep,
  title={Deep symbolic learning: Discovering symbols and rules from perceptions},
  author={Daniele, Alessandro and Campari, Tommaso and Malhotra, Sagar and Serafini, Luciano},
  journal={arXiv preprint arXiv:2208.11561},
  year={2022}
}

@InProceedings{pmlr-v202-barbiero23a,
  title = 	 {Interpretable Neural-Symbolic Concept Reasoning},
  author =       {Barbiero, Pietro and Ciravegna, Gabriele and Giannini, Francesco and Espinosa Zarlenga, Mateo and Magister, Lucie Charlotte and Tonda, Alberto and Lio, Pietro and Precioso, Frederic and Jamnik, Mateja and Marra, Giuseppe},
  booktitle = 	 {Proceedings of the 40th International Conference on Machine Learning},
  pages = 	 {1801--1825},
  year = 	 {2023},
  volume = 	 {202},
}

@InProceedings{satnet,
  title = 	 {{SATN}et: Bridging deep learning and logical reasoning using a differentiable satisfiability solver},
  author =       {Wang, Po-Wei and Donti, Priya and Wilder, Bryan and Kolter, Zico},
  booktitle = 	 {Proceedings of the 36th International Conference on Machine Learning},
  pages = 	 {6545--6554},
  year = 	 {2019},
  volume = 	 {97},
}

@article{r2n,
  title={Differentiable rule induction with learned relational features},
  author={Kusters, Remy and Kim, Yusik and Collery, Marine and Marie, Christian de Sainte and Gupta, Shubham},
  journal={arXiv preprint arXiv:2201.06515},
  year={2022}
}

@inproceedings{dr-net,
  title={Learning accurate and interpretable decision rule sets from neural networks},
  author={Qiao, Litao and Wang, Weijia and Lin, Bill},
  booktitle={Proceedings of the AAAI Conference on Artificial Intelligence},
  volume={35},
  pages={4303--4311},
  year={2021}
}

@inproceedings{net-dnf,
  title={Net-dnf: Effective deep modeling of tabular data},
  author={Katzir, Liran and Elidan, Gal and El-Yaniv, Ran},
  booktitle={International conference on learning representations},
  year={2020}
}

@inproceedings{rl-net,
  title={Rl-net: Interpretable rule learning with neural networks},
  author={Dierckx, Lucile and Veroneze, Rosana and Nijssen, Siegfried},
  booktitle={Pacific-Asia Conference on Knowledge Discovery and Data Mining},
  pages={95--107},
  year={2023},
  organization={Springer}
}

@inproceedings{mllp,
  title={Transparent classification with multilayer logical perceptrons and random binarization},
  author={Wang, Zhuo and Zhang, Wei and Wang, Jianyong and others},
  booktitle={Proceedings of the AAAI conference on artificial intelligence},
  volume={34},
  pages={6331--6339},
  year={2020}
}

@InProceedings{xu2018semantic,
  title = 	 {A Semantic Loss Function for Deep Learning with Symbolic Knowledge},
  author =       {Xu, Jingyi and Zhang, Zilu and Friedman, Tal and Liang, Yitao and Van den Broeck, Guy},
  booktitle = 	 {Proceedings of the 35th International Conference on Machine Learning},
  pages = 	 {5502--5511},
  year = 	 {2018},
  volume = 	 {80},
  series = 	 {Proceedings of Machine Learning Research},
  month = 	 {07},
  publisher =    {PMLR},
}

@inproceedings{winters2022deepstochlog,
  title={Deepstochlog: Neural stochastic logic programming},
  author={Winters, Thomas and Marra, Giuseppe and Manhaeve, Robin and De Raedt, Luc},
  booktitle={Proceedings of the AAAI Conference on Artificial Intelligence},
  volume={36},
  pages={10090--10100},
  year={2022}
}

@inproceedings{yang2020neurasp,
  title={NeurASP: Embracing neural networks into answer set programming},
  author={Yang, Zhun and Ishay, Adam and Lee, Joohyung},
  booktitle={29th International Joint Conference on Artificial Intelligence, IJCAI 2020},
  pages={1755--1762},
  year={2020},
  organization={International Joint Conferences on Artificial Intelligence}
}

@article{rocktaschel2017end,
  title={End-to-end differentiable proving},
  author={Rockt{\"a}schel, Tim and Riedel, Sebastian},
  journal={Advances in neural information processing systems},
  volume={30},
  year={2017}
}

@article{chang2020assessing,
  title={Assessing SATNet's ability to solve the symbol grounding problem},
  author={Chang, Oscar and Flokas, Lampros and Lipson, Hod and Spranger, Michael},
  journal={Advances in Neural Information Processing Systems},
  volume={33},
  pages={1428--1439},
  year={2020}
}

@article{topan2021techniques,
  title={Techniques for symbol grounding with satnet},
  author={Topan, Sever and Rolnick, David and Si, Xujie},
  journal={Advances in Neural Information Processing Systems},
  volume={34},
  pages={20733--20744},
  year={2021}
}

@article{deng2018gxnor,
  title={GXNOR-Net: Training deep neural networks with ternary weights and activations without full-precision memory under a unified discretization framework},
  author={Deng, Lei and Jiao, Peng and Pei, Jing and Wu, Zhenzhi and Li, Guoqi},
  journal={Neural Networks},
  volume={100},
  pages={49--58},
  year={2018},
  publisher={Elsevier}
}

@article{petersen2022deep,
  title={Deep differentiable logic gate networks},
  author={Petersen, Felix and Borgelt, Christian and Kuehne, Hilde and Deussen, Oliver},
  journal={Advances in Neural Information Processing Systems},
  volume={35},
  pages={2006--2018},
  year={2022}
}

@article{riegel2020logical,
  title={Logical neural networks},
  author={Riegel, Ryan and Gray, Alexander and Luus, Francois and Khan, Naweed and Makondo, Ndivhuwo and Akhalwaya, Ismail Yunus and Qian, Haifeng and Fagin, Ronald and Barahona, Francisco and Sharma, Udit and others},
  journal      = {CoRR},
  volume       = {abs/2006.13155},
  year         = {2020},
  url          = {https://arxiv.org/abs/2006.13155},
}

@inproceedings{tsamoura2021neurolog,
  title={Neural-symbolic integration: A compositional perspective},
  author={Tsamoura, Efthymia and Hospedales, Timothy and Michael, Loizos},
  booktitle={Proceedings of the AAAI conference on artificial intelligence},
  volume={35},
  pages={5051--5060},
  year={2021}
}

@article{dai2019abductive,
  title={Bridging machine learning and logical reasoning by abductive learning},
  author={Dai, Wang-Zhou and Xu, Qiuling and Yu, Yang and Zhou, Zhi-Hua},
  journal={Advances in Neural Information Processing Systems},
  volume={32},
  year={2019}
}

@article{huang2021fast,
  title={Fast abductive learning by similarity-based consistency optimization},
  author={Huang, Yu-Xuan and Dai, Wang-Zhou and Cai, Le-Wen and Muggleton, Stephen H and Jiang, Yuan},
  journal={Advances in Neural Information Processing Systems},
  volume={34},
  pages={26574--26584},
  year={2021}
}

@article{petersen2024convolutional,
  title={Convolutional differentiable logic gate networks},
  author={Petersen, Felix and Kuehne, Hilde and Borgelt, Christian and Welzel, Julian and Ermon, Stefano},
  journal={Advances in Neural Information Processing Systems},
  volume={37},
  pages={121185--121203},
  year={2024}
}

@article{adam,
  title={Adam: A method for stochastic optimization},
  author={Kingma, Diederik P},
  journal={arXiv preprint arXiv:1412.6980},
  year={2014}
}

@misc{adult,
  author       = {Becker, Barry and Kohavi, Ronny},
  title        = {{Adult}},
  year         = {1996},
  howpublished = {UCI Machine Learning Repository},
  note         = {{DOI}: https://doi.org/10.24432/C5XW20}
}

@misc{bank_marketing,
  author       = {Moro, S. and Rita, P. and Cortez, P.},
  title        = {{Bank Marketing}},
  year         = {2014},
  howpublished = {UCI Machine Learning Repository},
  note         = {{DOI}: https://doi.org/10.24432/C5K306}
}

@misc{banknote,
  author       = {Lohweg, Volker},
  title        = {{Banknote Authentication}},
  year         = {2012},
  howpublished = {UCI Machine Learning Repository},
  note         = {{DOI}: https://doi.org/10.24432/C55P57}
}

@misc{blogger,
  title        = {{BLOGGER}},
  year         = {2012},
  howpublished = {UCI Machine Learning Repository},
  note         = {{DOI}: https://doi.org/10.24432/C5HK6P}
}

@misc{chess,
  author       = {Bain, Michael and Hoff, Arthur},
  title        = {{Chess (King-Rook vs. King)}},
  year         = {1994},
  howpublished = {UCI Machine Learning Repository},
  note         = {{DOI}: https://doi.org/10.24432/C57W2S}
}

@misc{connect_4,
  author       = {Tromp, John},
  title        = {{Connect-4}},
  year         = {1995},
  howpublished = {UCI Machine Learning Repository},
  note         = {{DOI}: https://doi.org/10.24432/C59P43}
}

@misc{letRecog,
  author       = {Slate, David},
  title        = {{Letter Recognition}},
  year         = {1991},
  howpublished = {UCI Machine Learning Repository},
  note         = {{DOI}: https://doi.org/10.24432/C5ZP40}
}

@misc{magic,
  author       = {Bock, R.},
  title        = {{MAGIC Gamma Telescope}},
  year         = {2004},
  howpublished = {UCI Machine Learning Repository},
  note         = {{DOI}: https://doi.org/10.24432/C52C8B}
}

@misc{mushroom,
  author       = {Schlimmer, J.},
  title        = {{Mushroom}},
  year         = {1981},
  howpublished = {UCI Machine Learning Repository},
  note         = {{DOI}: https://doi.org/10.24432/C5959T}
}

@misc{nursery_76,
  author       = {Rajkovic, Vladislav},
  title        = {{Nursery}},
  year         = {1989},
  howpublished = {UCI Machine Learning Repository},
  note         = {{DOI}: https://doi.org/10.24432/C5P88W}
}

@misc{wine,
  author       = {Cortez, Paulo and Cerdeira, A. and Almeida, F. and Matos, T. and Reis, J.},
  title        = {{Wine Quality}},
  year         = {2009},
  howpublished = {UCI Machine Learning Repository},
  note         = {{DOI}: https://doi.org/10.24432/C56S3T}
}

@book{gumbel1954statistical,
  title={Statistical theory of extreme values and some practical applications: a series of lectures},
  author={Gumbel, Emil Julius},
  volume={33},
  year={1954},
  publisher={US Government Printing Office}
}

@inproceedings{jang2017categorical,
  title={Categorical Reparameterization with Gumbel-Softmax},
  author={Jang, Eric and Gu, Shixiang and Poole, Ben},
  booktitle={International Conference on Learning Representations},
  year={2017}
}

@article{jang1993anfis,
  title={ANFIS: adaptive-network-based fuzzy inference system},
  author={Jang, J-SR},
  journal={IEEE transactions on systems, man, and cybernetics},
  volume={23},
  number={3},
  pages={665--685},
  year={1993},
  publisher={IEEE}
}

@inproceedings{lazyPop,
  author={Zhou, R.W. and Quek, C.},
  booktitle={Proceedings of International Conference on Neural Networks (ICNN'96)}, 
  title={A pseudo outer-product based fuzzy neural network and its rule-identification algorithm}, 
  year={1996},
  pages={1156-1161 vol.2},
  doi={10.1109/ICNN.1996.549061}
}

@article{GSETSK,
title = {GSETSK: a generic self-evolving TSK fuzzy neural network with a novel Hebbian-based rule reduction approach},
journal = {Applied Soft Computing},
volume = {35},
pages = {29-42},
year = {2015},
issn = {1568-4946},
doi = {https://doi.org/10.1016/j.asoc.2015.06.008},
author = {Ngoc Nam Nguyen and Weigui Jair Zhou and Chai Quek},
}

@article{NFN-survey,
title = {Recent advances in neuro-fuzzy system: A survey},
journal = {Knowledge-Based Systems},
volume = {152},
pages = {136-162},
year = {2018},
issn = {0950-7051},
doi = {https://doi.org/10.1016/j.knosys.2018.04.014},
author = {K.V. Shihabudheen and G.N. Pillai},
}

@inproceedings{
gao2025differentiable,
title={Differentiable Rule Induction from Raw Sequence Inputs},
author={Kun Gao and Katsumi Inoue and Yongzhi Cao and Hanpin Wang and Yang Feng},
booktitle={The Thirteenth International Conference on Learning Representations},
year={2025}
}

@inproceedings{
yousefi2025mind,
title={Mind the Gap: Removing the Discretization Gap in Differentiable Logic Gate Networks},
author={Shakir Yousefi and Andreas Plesner and Till Aczel and Roger Wattenhofer},
booktitle={The Thirty-ninth Annual Conference on Neural Information Processing Systems},
year={2025}
}

@article{perreault2025neural,
  title={Neural Logic Networks for Interpretable Classification},
  author={Perreault, Vincent and Inoue, Katsumi and Labib, Richard and Hertz, Alain},
  journal={arXiv preprint arXiv:2508.08172},
  year={2025}
}

\ifarxiv
\newpage
\section{Gödel Trick}
\label{app:godeltrick}
In this appendix, we provide a recap of the \emph{Gödel trick with categorical variables}, adapting its original presentation in~\cite{Godel_trick} to our setting with fuzzy truth values in $[0,1]$ discretized to $\{0,1\}$ via a $0.5$ threshold (instead of real values discretized to $\{-1,1\}$ via the sign function). We also adapt the proof in~\cite[Section 8]{Godel_trick}, showing that the Gödel Trick is equivalent to the classical \emph{Gumbel-max trick} \cite{gumbel1954statistical}, of which the widely used \emph{Gumbel-softmax trick}~\cite{jang2017categorical} is a smooth approximation.

%Within the Logic of Hypotheses framework, the Gödel Trick allows end-to-end differentiable training while guaranteeing lossless binarization of gates at test time. By combining:
%\begin{itemize}
%    \item Gödel semantics for logical consistency after thresholding,
%    \item Gumbel perturbations for stochastic exploration,
%    \item temperature-controlled sigmoid for smooth optimization,
%\end{itemize}
%we obtain a principled mechanism to learn discrete subformula selections in a continuous space, bridging symbolic rule search and gradient-based learning.

%Recall that under Gödel fuzzy semantics, conjunction and disjunction are interpreted as:
%\[
%x \wedge y = \min(x,y), \quad x \vee y = \max(x,y), \quad \neg x = 1-x.
%\]
%A key property of Gödel logic is that thresholding all truth values yields a Boolean interpretation consistent with the Boolean evaluation of the same formula (Theorem \ref{th:binarization}). This allows us to \emph{discretize} both intermediate truth values and learned weights without altering the final Boolean prediction. Hence, after training, every choice operator selects exactly one branch by setting a single gate $w_i$ above $0.5$ and the rest below.

\paragraph{Gödel Trick Recap.}
Direct gradient descent on Gödel logic is prone to stalling in local minima. The Gödel Trick counteracts this by injecting noise.
In our formulation, it takes a set of logits $z_i$ (one per candidate branch of a choice operator), and applies the following three steps, depicted in Figure \ref{fig:weights}:
\begin{enumerate}
\item \textbf{Noise addition (only during training):} 
$z'_i := z_i + n_i$, with $n_i \stackrel{\text{i.i.d.}}{\sim}  \mathrm{Gumbel}(0,\beta)$.
\item \textbf{Re-centering:} $
z''_i := z'_i - \bar{z}'$, where $\bar{z}'$ is the mean of the two largest perturbed logits.
\item \textbf{Sigmoid application:} $w_i := \sigma\!\left( \frac{z''_i}{T} \right)$, 
where $T>0$ is a temperature hyperparameter. 
\end{enumerate}
At test time, we can binarize the weights $w_i$ at $0.5$, ensuring that exactly one $w_i$ equals $1$ and the rest are $0$. Thanks to Theorem \ref{th:binarization}, this discretization does not alter the final Boolean predictions of the model.

%\begin{wrapfigure}{l}{0.5\textwidth}
\begin{figure}[ht] \centering
\begin{tikzpicture}[>=latex,scale=1.5,yscale=-0.95]
\tikzset{tick/.style={circle,fill=black,inner sep=1pt}}
%------------------------------------------------
%  Nodes
%------------------------------------------------
\node (A) at (3.35,7.0) {};
\node[tick,label=below:$z_c$] (B) at (4.25,7.0) {};
\node[tick,label=below:$0$] (C) at (5.0,7.0) {};
\node[tick,label=below:$z_a$] (D) at (5.6,7.0) {};
\node[tick,label=below:$z_b$] (E) at (6,7.0) {};
\node (F) at (7.2,7.0) {};
\node  (G) at (3.35,6.0) {};
\node[tick,label=below:$z'_c$] (H) at (4.75,6.0) {};
\node[tick,label=below:$\,\,z'_a$] (I) at (5.4,6.0) {};
\node[tick,label=below:$\bar z'$] (J) at (5.9,6.0) {};
\node[tick,label=below:$z'_b$] (K) at (6.4,6.0) {};
\node (L) at (7.2,6.0) {};
\node (M) at (3.35,5.0) {};
\node[tick,label=below:$z''_c$] (N) at (3.85,5.0) {};
\node[tick,label=below:$z''_a$] (O) at (4.5,5.0) {};
\node[tick] (P) at (5.0,5.0) {};
\node[tick,label=below:$z''_b$] (Q) at (5.5,5.0) {};
\node (R) at (7.2,5.0) {};
\node (S) at (3.35,4.5) {};
\node (T) at (7.2,4.5) {};
\node (U) at (3.25,3) {};
\node (V) at (7.2,3) {};
\node (W) at (5.0,2.5) {};
\node (Z) at (5.0,4.85) {};

%  Named points on the curve aligned with N, O, P and Q
\node[tick] (Nsig) at (3.85,4.325) {};
\node[tick] (Osig) at (4.5,4.06) {};
%\node[tick, label=left:$0.5$] (Psig2) at (5,4) {};
\node[tick] (Qsig) at (5.5,3.44) {};

\node[tick, label=right:$w_c$] (Nsig2) at (5,4.325) {};
\node[tick,label=right:$w_a$] (Osig2) at (5,4.06) {};
\node[tick, label=left:$w_b$] (Qsig2) at (5,3.44) {};

%------------------------------------------------
%  Sigmoid
%------------------------------------------------
\draw[thick,domain=3.45:6.85,samples=150,smooth,variable=\x]
      plot (\x,{4.5 - 1.5/(1+exp(-1.75*(\x-5))) });

%------------------------------------------------
%  Edges
%------------------------------------------------
\draw[->] (A) -- (F);
\draw[->] (G) -- (L);
\draw[->] (M) -- (R);
\draw[->] (Z) -- (W);
\draw[dashed] (C) -- (P);
\draw[->] (B) -- (H);
\draw[->] (D) -- (I);
\draw[->] (E) -- (K);
\draw[dashed,->] (J) -- (P);
\draw[->] (I) -- (O);
\draw[->] (K) -- (Q);
\draw[->] (H) -- (N);
\draw[->] (S) -- (T);
\draw[dashed] (U) -- (V);

\draw[-] (N) -- (Nsig);
\draw[-] (O) -- (Osig);
\draw[-] (Q) -- (Qsig);

\draw[->] (Nsig) -- (Nsig2);
\draw[->] (Osig) -- (Osig2);
\draw[->] (Qsig) -- (Qsig2);

% Text labels for the horizontal bands (right side)
\node[anchor=west] at (7.4,6.4) {$z'_i := z_i + n_i$};
\node[anchor=west] at (7.4,6.7) {$n_i \sim G(0,\beta)$};
\node[anchor=west] at (7.4,5.5) {$z''_i := z'_i - \bar z'$};
\node[anchor=west] at (7.4,4) {$w_i := \sigma(z''_i/T)$};
\end{tikzpicture}
\caption{Given a choice operator $[a,b,c]$, the figure illustrates the differentiable three-stage procedure that turns the raw, real-valued logits $z_i$ attached to each candidate sub-formula into logical gates $w_i\in[0,1]$. From the bottom up: (1) during training, i.i.d.\ Gumbel noise $n_i\sim\mathrm{Gumbel}(0,\beta)$ is added to each logit; (2) the mean $\bar z'$ of the two largest perturbed logits is subtracted from all values; (3) temperature-scaled sigmoid function is applied.  %When the $w_i$'s are later discretized by the $0.5$ threshold, exactly one gate remains open ($w_i>0.5$) and all others close, selecting a single sub-formula.
} \label{fig:weights}
\end{figure}
%\end{wrapfigure}

\paragraph{Equivalence to the Gumbel-max Trick.}
The \emph{Gumbel-max trick} is a reparameterization
method for sampling from a categorical distribution. Let the categorical distribution be defined by the probabilities $\pi \coloneqq softmax(\theta)$. Then, the probability that $\arg\max_{i} \left( \theta_i + g_i \right) = j$, where $g_i \sim Gumbel(0,1)$, is precisely equal to $\pi_j$. Hence, \[\onehot(\arg\max_{i} \left( \theta_i + g_i \right))\] produces an exact sample from the categorical distribution with probabilities $\pi$. This is how the Gumbel-max trick works, and the Gumbel-softmax trick is a differentiable approximation which substitutes $\onehot\circ\arg\max$ with $softmax$.

Setting $\theta_i := z_i / \beta$ and $g_i := n_i / \beta$, the perturbation step of the Gödel Trick becomes:
\[
z'_i = \beta \left( \theta_i + g_i \right)
\]
The centering step subtracts the mean of the two largest perturbed logits, which does not affect the $\arg\max$, since it simply shifts all values by the same constant. Similarly, also the multiplication by a positive constant and the application of a monotonically increasing function such as the sigmoid do not affect the $\arg\max$.  Hence:
\[
\arg\max_i w_i = \arg\max_i z''_i = \arg\max_i z'_i = \arg\max_i (\theta_i + g_i)
\]
Since the maximum weight is by construction the only one surpassing the threshold value $0.5$, discretizing the weights $w_j$ yields precisely the same effect as $\onehot(\arg\max_i(w_i))$. Thus, we can conclude that 
\[
\rho(w_j) = \onehot(\arg\max_{i} \left( \theta_i + g_i \right))_j
\] 
for every $j$.
Moreover, because we are using Gödel semantics and Theorem \ref{th:binarization} applies, the thresholding of the weights is implicit when considering binarized output predictions of the entire model. Hence, in our context, the Gödel trick is equivalent to the Gumbel-max trick, which is more precise than the Gumbel-softmax one. 

\section{Counterexample for Product Fuzzy Logic}\label{app:counterexample}
Let us consider the disjunctive compilation of $[a, b]$:
\begin{equation*}\label{eq:counterexample}
    (w_a \land a) \lor (w_b \land b)
\end{equation*}
and interpret it with the following fuzzy values:
$\mathcal{I}(w_a) = 0.6$, $\mathcal{I}(a)=0.7$, $\mathcal{I}(w_b)=0.4$ and $\mathcal{I}(b)=0$.
Using product fuzzy logic, 
\[
\mathcal{I}((w_a \land a) \lor (w_b \land b)) = 0.6 * 0.7 = 0.42 < 0.5
\]
whose discretized Boolean value $\rho_{0.5}(0.42)=0$ (False).
%is $0$ (False).
On the other hand, if $\mathcal{B}$ is the Boolean discretization of $\mathcal{I}$, so that $\mathcal{B}(w_a) = 1$, $\mathcal{B}(a)=1$, $\mathcal{B}(w_b)=0$ and $\mathcal{B}(b)=0$, then
\[
\mathcal{B}((w_a \land a) \lor (w_b \land b)) = 1
\]
This means that the discretized explanation disagrees with the behavior of the network.  Analogous counterexamples can be built for Łukasiewicz logic. Instead, with Gödel logic, the Boolean interpretation  is always in accordance with the fuzzy truth values (Theorem \ref{th:binarization}):
\[
\mathcal{I}((w_a \land a) \lor (w_b \land b)) {=} \max(\min(0.6,0.7),\min(0.4,0)) {=} 0.6
\]

\section{Disjunctive vs Conjunctive Compilations} \label{app:versus}

Recall that a choice node $[\Phi_1, \dots, \Phi_n]$ can be compiled in two dual ways:
\begin{itemize}
  \item Disjunctive compilation: $\ \bigvee_{i=1}^n (w_i \wedge \Phi_i)$
  \item Conjunctive compilation: $\ \bigwedge_{i=1}^n (\lnot w_i \lor \Phi_i)$
\end{itemize}

Duality comes from De Morgan's laws: the negation of the disjunctive compilation is equivalent to the conjunctive compilation on the negated subformulas, and vice versa. Indeed,
\[
\neg \bigvee_{i=1}^n \ w_i \land \Phi_i \equiv  \bigwedge_{i=1}^n \ \neg w_i \lor \neg \Phi_i  \]
and 
\[ \neg \bigwedge_{i=1}^n \ \neg w_i \lor \Phi_i \equiv \bigvee_{i=1}^n \ w_i \land \neg \Phi_i  
\]

Under Gödel semantics, conjunction and disjunction are interpreted as $\min$ and $\max$ respectively, so the two translations become
\begin{align*}
\text{Disj.}: \quad & \max_i \min(w_i, \Phi_i) \\
\text{Conj.}: \quad & \min_i \max(1 - w_i, \Phi_i)
\end{align*}

\paragraph{Case Study: Selecting Reliable Rules.} Equation \eqref{eq:select_rules} proposes the following LoH model, which is also used inside conjunctive neurons of the type in \eqref{eq:neurons}:
\[
\bigwedge_{i=1}^n [r_i,\, \top] \tag{3}
\]
Notice that the two weights in the compilation of a choice operator choosing only among two subformulas must sum to one.\footnote{Indeed, let $z_1$ and $z_2$ be the stored parameters. By design, the weights $w_1$ and $w_2$ are obtained applying the sigmoid function to $z_1 - \frac{z_1+z_2}{2} =  \frac{z_1-z_2}{2}$ and $z_2 - \frac{z_1+z_2}{2} =  \frac{z_2-z_1}{2}$ (respectively). The two logit values are opposite to each other, so the weights $w_1 = \sigma(\frac{z_1-z_2}{2T})$ and $w_2 = \sigma(\frac{z_2-z_1}{2T})$ must sum to $1$.} Hence, we will write $w_i$ and $1-w_i$ for the (respective) weights of $r_i$ and $\top$ in $[r_i,\, \top]$.

With conjunctive compilation, %each bracket yields $\max(1 - w_i, r_i)$. Since $\text{True} = 1$, the constant branch disappears, giving:
\eqref{eq:select_rules} yields:
\begin{multline*}
 \min_i (\min(\max(1 - w_i, r_i), \max(w_i, \top))) = \\ = \min_i (\max(1 - w_i, r_i)) 
\end{multline*}
%If $r_i$ is \emph{kept} ($w_i > \tfrac{1}{2}$), it contributes $r_i$; otherwise it contributes 1—matching the intended semantics: “select any subset of rules.”
On the other hand, with disjunctive compilation, it becomes:
\begin{multline*}
\min_i (\max(\min(w_i, r_i), \min(1-w_i, \top))) = \\ = \min_i (\max(\min(w_i, r_i), 1-w_i))
\end{multline*}
%Even if $r_i = 0$, $w_i' \ge \tfrac{1}{2}$ for discarded rules, so each clause evaluates to at least 0.5—\textbf{blinding the outer conjunction}. The selector can no longer deactivate a faulty rule.

Thus, \eqref{eq:select_rules} simplifies more under conjunctive compilation. This simplification has an important effect on the training dynamics. Indeed, when a rule $r_i$ is already (almost) satisfied for a specific example (i.e., $r_i\approx1$), also the inner $\max(1-w_i,r_i)$ becomes $\approx1$. Hence, the outer $\min_i$ ignores that index and concentrates on a rule whose truth value is \emph{smaller} (if any). Then, it is the weight associated to that less-satisfied rule that will receive gradient signal for updating. Intuitively, when particular data gives no evidence for preferring $r_i$ over $\top$ (because their truth values are the same), the network also receives no signal to change the associated gate $w_i$.
By contrast, $\max(\min(w_i,r_i),1-w_i)$  evaluates to $\max(w_i,1-w_i)$ when $r_i\approx 1$. Consequently, the loss may push $w_i$ upwards or downwards even when the data provide no information for choosing between $r_i$ and $\top$. These unwanted updates can slow convergence and may bias the learned subset of rules. %Hence, the conjunctive compilation’s simplification is valuable: it automatically masks weight updates whenever the corresponding rule already fits the current sample, directing learning to the truly uncertain choices and yielding faster, cleaner training dynamics

\paragraph{Empirical Evaluation.} We conducted some experiments validating the previous findings and extending to different LoH models. The experiments are conducted in the following way. We consider $10$ propositional variables and randomly generate some ground-truth clauses and some additional clauses, of width ranging from $2$ to $5$. For any of the $2^{10}$ possible Boolean interpretations, a label is produced using the ground-truth clauses as rules. Then the dataset is divided into $75\%$ for training and $25\%$ for evaluation. An LoH model is trained to select some rules among the ground-truth + additional rules. For any considered number of ground-truth clauses and additional clauses, the same experiment is repeated $10$ times. 
For each execution, the test-set F$1$ score is recorded, together with the number of optimization steps to convergence. The criterion we use for deciding convergence is the following: either $100\%$ accuracy is achieved, or there is no change in accuracy for $64$ consecutive steps, or a limit of $64$ epochs ($384$ steps) is reached.
The values of the hyperparameters were fixed: $128$ as batch size, $0.15$ as learning rate, $1$ as temperature, and $Gumbel(0,1)$ noise.

Figure \ref{fig:conj-vs-disj} reports the plots with the experiments' results for different models. Subfigure \ref{fig:conj-vs-disj-cnf} refers to simple rule selection, with LoH models as in \eqref{eq:select_rules}. %The number of ground-truth clauses took values $5$, $10$, $15$, $20$ and that of additional clauses $10$, $17$, $24$, $31$. 
Subfigure \ref{fig:conj-vs-disj-dnf} does the same, but in the dual set-up: both ground-truth and additional clauses $c_i$ are \emph{conjunctive} clauses, and the LoH model $\bigvee_{i=1}^n [c_i,\, \bot]$ learns a disjunction of them. 
Finally, Subfigure \ref{fig:conj-vs-disj-fixed} considers again rules made of (disjunctive) clauses. However, each of the $m$ ground-truth clauses is placed on a different set $\{r_{i,1},\,r_{i,2},\dots,r_{i,k{+}1}\}$, together with $k$ additional clauses. Then, we use the model in \eqref{eq:rules_fixed}, selecting one rule per set. 

These plots corroborate our suggestion to use conjunctive compilation inside conjunctions (Subfigures \ref{fig:conj-vs-disj-cnf} and \ref{fig:conj-vs-disj-fixed}), and disjunctive compilation inside disjunctions (Subfigure \ref{fig:conj-vs-disj-dnf}). Indeed, such compilation choices achieved better results both in terms of final accuracy and of convergence speed.

\begin{figure}[ht]
\centering
\begin{subfigure}{\linewidth}
    \centering
    \includegraphics[width=\textwidth]{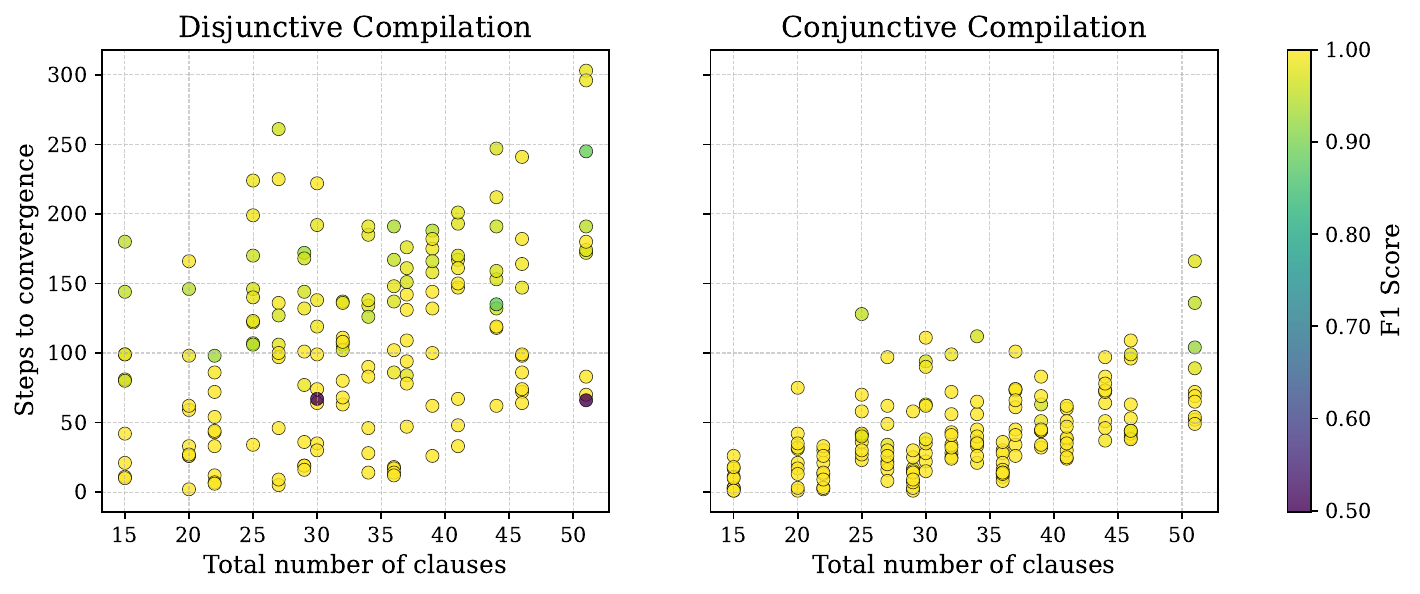}
    \caption{With LoH model $\bigwedge_{i=1}^n [r_i,\, \top]$}
    \label{fig:conj-vs-disj-cnf}
  \end{subfigure}

  \vspace{1em} 

\begin{subfigure}{\linewidth}
    \centering
    \includegraphics[width=\textwidth]{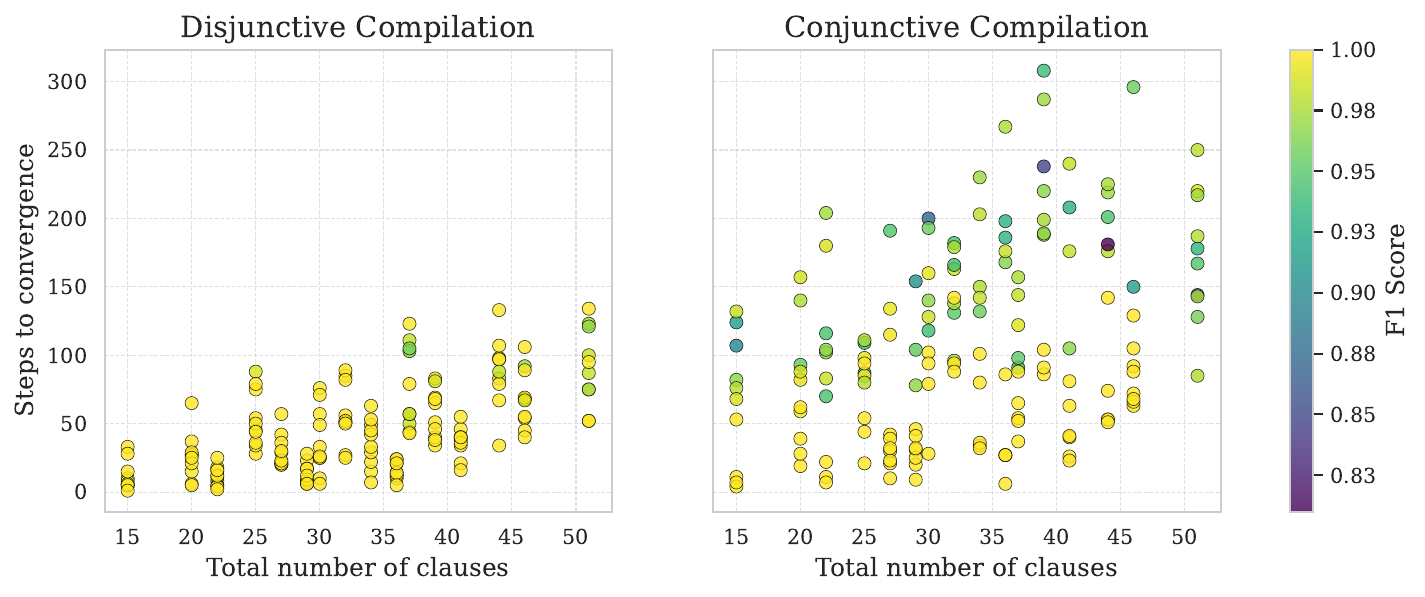}
    \caption{With LoH model $\bigvee_{i=1}^n [c_i,\, \bot]$}
    \label{fig:conj-vs-disj-dnf}
  \end{subfigure}

  \vspace{1em} 

\begin{subfigure}{\linewidth}
    \centering
    \includegraphics[width=\textwidth]{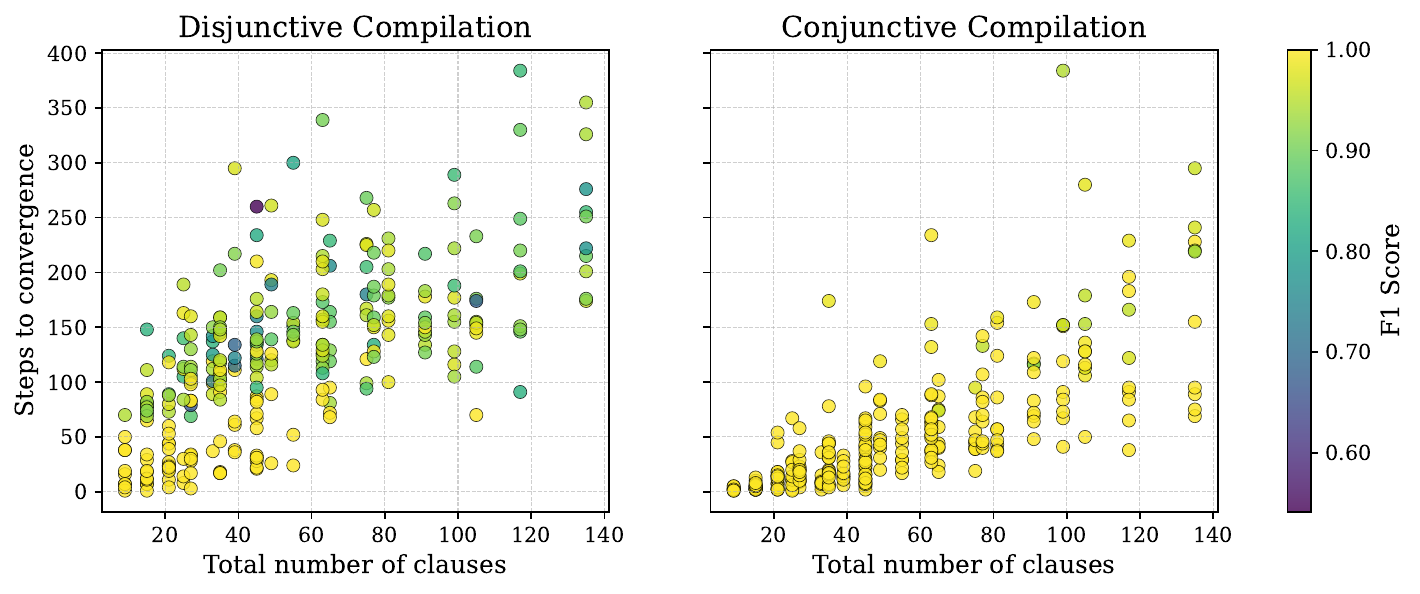}
    \caption{With LoH model $\bigwedge_{i=1}^{m} [r_{i,1},\,r_{i,2},\dots,r_{i,k{+}1}]$}
    \label{fig:conj-vs-disj-fixed}
  \end{subfigure}
\caption{Comparison of disjunctive and conjunctive compilations.}\label{fig:conj-vs-disj}
\end{figure}

%\subsection*{B.3\quad Practical Guidelines}

%\begin{itemize}
%  \item \textbf{Inside a conjunction}: use conjunctive compilation.
%  \item \textbf{Inside a disjunction}: use disjunctive compilation.
%  \item \textbf{Under negation:} flip the choice (by De Morgan).
%\end{itemize}

%Following these rules ensures the compiled network remains faithful to the intended symbolic semantics, and unused options are rendered semantically inert.

\section{Further Experiments with Artificial Data}

% \redNote{Paragrafo tolto che ripeteva l'appendice B.}
%We performed experiments on artificial data for different LoH models discussed in Section \ref{sec:unifying}. The outline is already introduced in Appendix \ref{app:versus}. In particular, we consider $10$ propositional variables and randomly generate a ground-truth formula and some additional subformulas. For any of the $2^{10}$ possible Boolean interpretations, a label is produced using the ground-truth formula. $75\%$ of this data is used for training; the remaining $25\%$ for evaluation. A LoH model is trained to select some subformulas from those of the ground-truth plus the additional ones. For each execution, the test-set F$1$ score is recorded after every optimization step. We use early-stopping (as discussed in Appendix \ref{app:versus}), and the hyperparameters values are fixed: $128$ as batch size, $0.15$ as learning rate, $1$ as temperature, and $Gumbel(0,1)$ noise. Finally, we use conjunctive compilation inside conjunctions, and disjunctive compilation inside disjunctions, as we suggested. 

%\todo{Figure \ref{fig:3plots}  and Figure \ref{fig:definite_clauses}.}
Figure~\ref{fig:3plots} reprises the experiments of Appendix~\ref{app:versus}.  However, it considers only the best compilations (disjunctive for \ref{fig:dnf}, and conjunctive for \ref{fig:cnf} and \ref{fig:fixed}), while  highlighting the influence of different experimental factors. In particular, the middle column shows how both F$1$ score and convergence speed have a strong negative correlation with the number of ground-truth clauses. Clearly, the more clauses in the ground truth, the more need to be selected for a perfect score, and the more difficult the learning. However, since the ground-truth clauses are independently generated, their number is also strongly correlated with data imbalance. %: positively for the disjunction of conjunctive clauses (Subfigure \ref{fig:dnf}) and negatively otherwise (Subfigures \ref{fig:cnf} and \ref{fig:fixed}). 
As shown in the first column, learning a conjunction of clauses suffers most when there are too few positive samples (i.e., samples in which the ground-truth formula is satisfied). Conversely, a shortage of negative samples drives poorer performance when learning a disjunction.
In contrast, the number of additional, ``misleading'' clauses is not as impactful as the ground-truth, as evidenced by the third column. This is true at least when the number of additional clauses is comparable to that of the ground truth.\footnote{Regarding Subfigure \ref{fig:fixed}, notice that the \emph{total} number of added clauses is a \emph{multiple} of the number of additional clauses \emph{per set}.}

\begin{figure}[ht]
\centering
\begin{subfigure}{\linewidth}
    \centering
    \includegraphics[width=\textwidth]{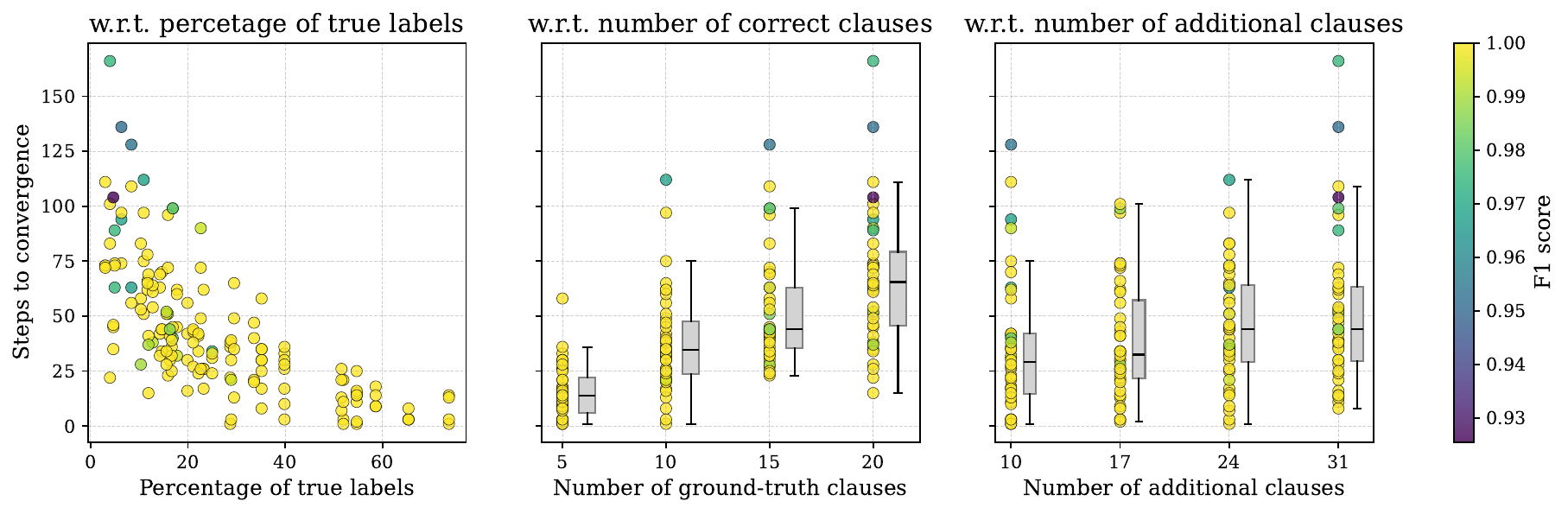}
    \caption{With LoH model $\bigwedge_{i=1}^n [r_i,\, \top]$}
    \label{fig:cnf}
  \end{subfigure}

  \vspace{1em} 

\begin{subfigure}{\linewidth}
    \centering
    \includegraphics[width=\textwidth]{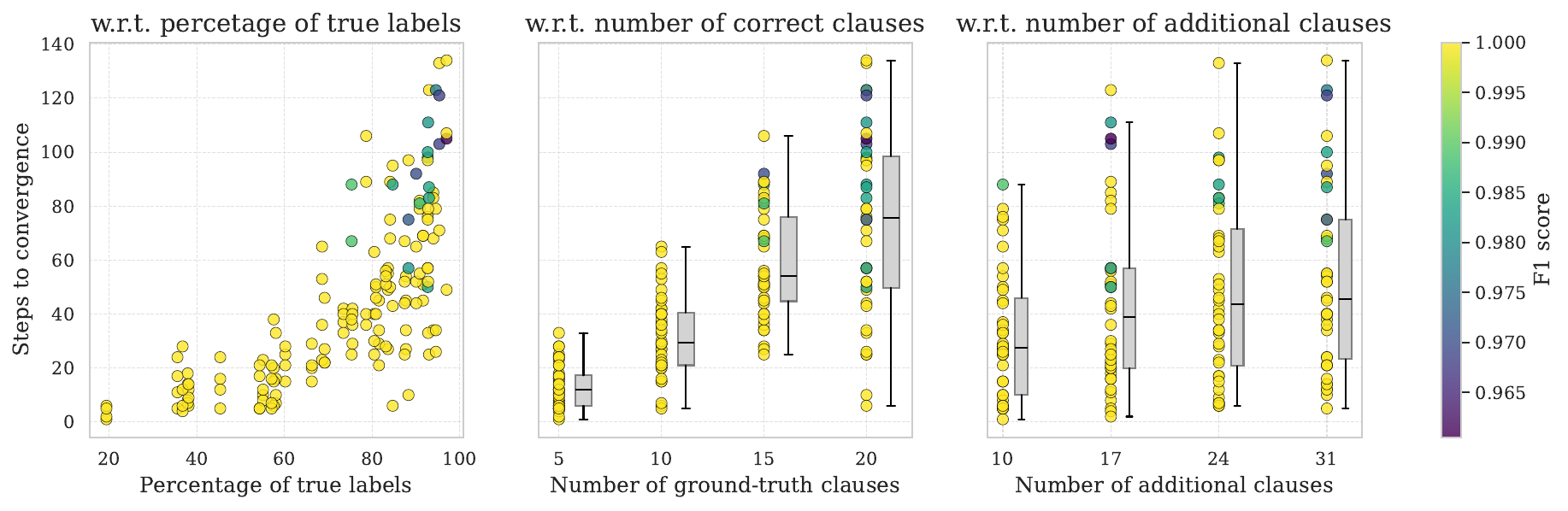}
    \caption{With LoH model $\bigvee_{i=1}^n [c_i,\, \bot]$}
    \label{fig:dnf}
  \end{subfigure}

  \vspace{1em} 

\begin{subfigure}{\linewidth}
    \centering
    \includegraphics[width=\textwidth]{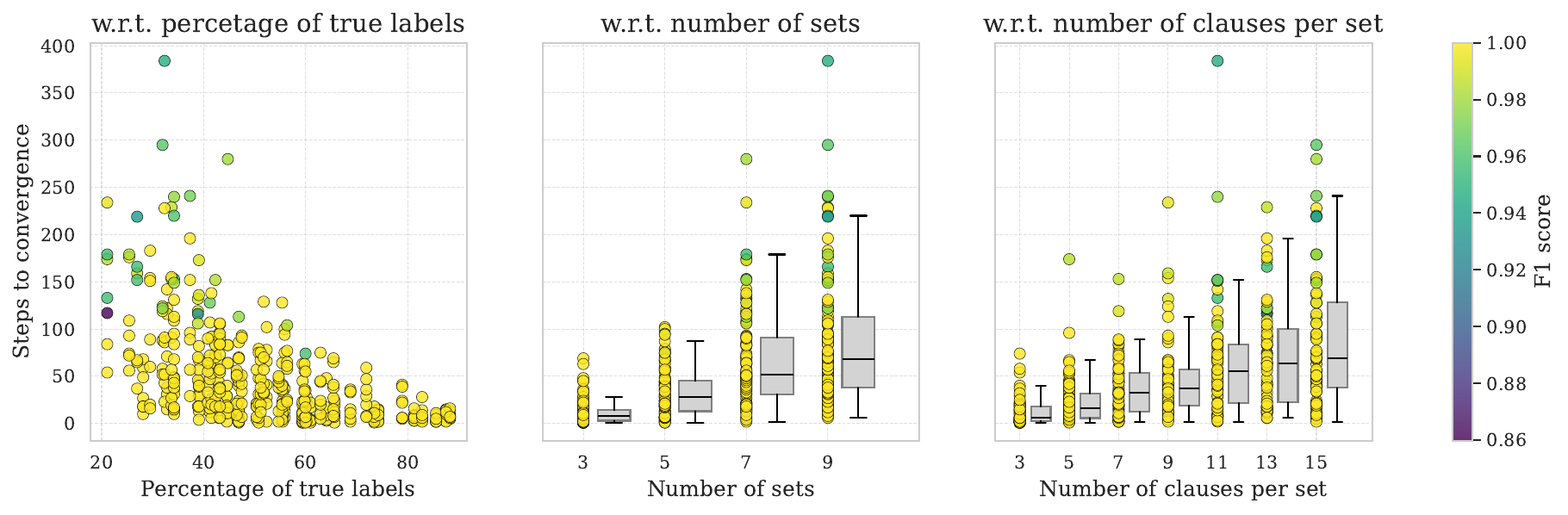}
    \caption{With LoH model $\bigwedge_{i=1}^{m} [r_{i,1},\,r_{i,2},\dots,r_{i,k{+}1}]$}
    \label{fig:fixed}
  \end{subfigure}
\caption{Clauses selection performance w.r.t. percentage of true labels, number of ground-truth clauses and number of additional clauses.}\label{fig:3plots}
\end{figure}

\begin{comment}
\begin{figure}
    \centering
    \includegraphics[width=0.6\linewidth]{Plots/fig_definite_clauses.png}
    \caption{Average training curves---on the same data---of LoH formulas following different templates (Equations \eqref{eq:5_definite}, \eqref{eq:definite_clauses2} and \eqref{eq:5_clauses}).}
    \label{fig:definite_clauses}
\end{figure}
\end{comment}

\section{Wildfire Risk Assessment Task Details}\label{app:wildfire}

%\paragraph{Variables and ground–truth.}
%The task is defined over nine propositional variables. Two of them are \emph{visual} concepts, to be formed by a CNN from the images: $\DenseForest$ and $\DryVegetation$. The remaining seven are \emph{non–visual} features provided in tabular form: $\LowHumidity$, $\HighTemperature$, $\RainedRecently$, $\StrongWind$, $\LightningsFrequent$, $\LowHumanActivity$ and $\PowerLinesNearby$.
%The ground–truth label $\WildFireRisk$ is obtained by combining three interpretable intermediate rules: see equations \ref{eq:wildfire}.

\paragraph{Dataset generation.}
The dataset contains 2048 synthetic RGB images of size
$256\times 256$. For each image, we first sample Boolean assignments to the visual concepts ($\DenseForest$ and $\DryVegetation$), deciding whether the scene should contain dense forest and whether the vegetation should be dry. A simple drawing routine then renders forest-like patches, and the color scheme distinguish greener from dry vegetation. The image generation process may also produce additional elements, such as houses or rivers, that are random artifact not correlated with the wildfire risk. %Examples of the generated images are available in Figure \ref{fig:wildfire_images}.

\begin{comment}
\begin{figure}[t]
  \centering
  % Dense forest, NOT dry vegetation
  \begin{subfigure}[b]{0.23\textwidth}
    \centering
    \includegraphics[width=\linewidth]{Plots/image_0004.png}
    \caption{\small $\DenseForest, \neg \DryVegetation$}
  \end{subfigure}\hfill
  % Dense forest, dry vegetation
  \begin{subfigure}[b]{0.23\textwidth}
    \centering
    \includegraphics[width=\linewidth]{Plots/image_0001.png}
    \caption{\small $\DenseForest, \DryVegetation$}
  \end{subfigure}%\hfill
  \\
  % No dense forest, NOT dry vegetation
  \begin{subfigure}[b]{0.23\textwidth}
    \centering
    \includegraphics[width=\linewidth]{Plots/image_0003.png}
    \caption{\small $\neg \DenseForest, \neg \DryVegetation$}
  \end{subfigure}\hfill 
  % No dense forest, dry vegetation
  \begin{subfigure}[b]{0.23\textwidth}
    \centering
    \includegraphics[width=\linewidth]{Plots/image_0000.png}
    \caption{\small $\neg \DenseForest, \DryVegetation$}
  \end{subfigure}

  \caption{Examples of synthetic wildfire scenes for the four possible combinations of $\DenseForest$ and $\DryVegetation$.} \label{fig:wildfire_images}
\end{figure}
\end{comment}

Independently, we sample Boolean values for the seven non-visual features
($\LowHumidity$, $\PowerLinesNearby$, etc.), as iid $Bernoulli(0.5)$ random variables. The ground-truth wildfire-risk label for each sample is obtained by evaluating Equation~\eqref{eq:ground_truth} on the full 9-dimensional Boolean vector. The images,  the non-visual features and the risk labels are stored in a PyTorch dataset and split into training and test sets (75\%/25\%).

\paragraph{Model Architectures and training.}
The neurosymbolic models integrate a perception and a reasoning module, which are trained end-to-end:
\begin{itemize}
    \item \textbf{Perception Module (CNN):} A three-layer Convolutional Neural Network processes $3 \times 256 \times 256$ RGB images to extract the visual predicates $\DenseForest$ and $\DryVegetation$. Each block consists of a convolution (channels: $3 \to 16 \to 16 \to 16$), ReLU activation, MaxPool ($2 \times 2$), and Dropout ($p=0.15$). A final classification head outputs fuzzy values for the two visual concepts, through a $sigmoid$ activation.
    \item \textbf{Reasoning Module (LoH):} This layer accepts the visual predictions alongside the non-visual boolean features. It comes in the four different knowledge regimes discussed in Section~\ref{sec:unifying}.
\end{itemize}

Training uses standard binary cross–entropy on the wildfire label, with Adam optimizers for both the logical and perceptual parts. The CNN learning rate is $8\cdot 10^{-4}$ and the LoH learning rate is $8\cdot 10^{-2}$. For each knowledge regime, the training is performed $20$ times with different random seeds. Analogously for $\partial$ILP.

\section{Comparison of Different Templates}\label{app:ablation}
Let us consider the following ground-truth CNF formula $\phi$ made of $5$ definite clauses of width $3$:
\begin{multline*}
(\neg v_3 \lor \neg v_8 \lor v_7) \land (\neg v_{10} \lor \neg v_3 \lor v_4) \land (\neg v_1 \lor \neg v_9 \lor v_{10}) \\ \land (\neg v_2 \lor \neg v_6 \lor v_8) \land (\neg v_4 \lor \neg v_3 \lor v_5)
\end{multline*}
For any of the $2^{10}$ possible Boolean interpretations of the propositional variables $v_1,\dots,v_{10}$, a ground-truth label is produced using $\phi$. This dataset is divided into $75\%$ for training and $25\%$ for evaluation, and each LoH model is trained to construct $5$ clauses.
The values of the hyperparameters  are fixed: $128$ as batch size, $0.15$ as learning rate, $1$ as temperature, and $Gumbel(0,1)$ noise.
Figure~\ref{fig:definite_clauses} compares the average learning curves of LoH models adhering to different templates:
\newcommand{\tland}{\mathrel{\mkern-5mu\land\mkern-5mu}}
\begin{itemize}
\item ``5 clauses'' conjoins five copies of
%\begin{equation}\label{eq:5_clauses}
$    \bigvee_{i=1}^{10} [\neg v_i, v_i, \bot] $ %\;\lor\; \bigvee_{i=1}^{10} [v_i, \bot] 
%\end{equation}
\item ``5 definite clauses'' uses  the conjunction of five copies of formula \eqref{eq:definite_clauses} with $n{=}10$, i.e.,
%\begin{equation}\label{eq:definite_clauses2}
$    \bigvee_{i=1}^{10} [\neg v_i, \bot]
    \;\lor\; 
    [v_1,\dots,v_{10}] $
%\end{equation}
\item ``5 definite clauses of width $3$'' uses the conjunction of five copies of
%\begin{equation}\label{eq:5_definite}
$  [\neg v_1, \dots, \neg v_{10}] \lor [\neg v_1, \dots, \neg v_{10}] \lor [v_1, \dots, v_{10}] $
%\end{equation}
\item ``5 definite clauses with heads given'' uses
$(\bigvee_{i=1}^{10} [\neg v_i, \bot] \lor v_7) \tland  (\bigvee_{i=1}^{10} [\neg v_i, \bot] \lor v_4) \tland (\bigvee_{i=1}^{10} [\neg v_i, \bot] \lor v_{10}) \tland (\bigvee_{i=1}^{10} [\neg v_i, \bot] \lor v_8) \tland (\bigvee_{i=1}^{10} [\neg v_i, \bot] \lor v_5)$ %\;
\item ``5 definite clauses with first given'' conjoins $\neg v_3 {\lor} \neg v_8 {\lor} v_7$ with four copies of $\bigvee_{i=1}^{10} [\neg v_i, \bot]
    \lor
    [v_1,\dots,v_{10}] $
\end{itemize}

%\begin{wrapfigure}{l}{0.48\textwidth}
\begin{figure}[ht]
    \centering
    \includegraphics[width=0.92\linewidth]{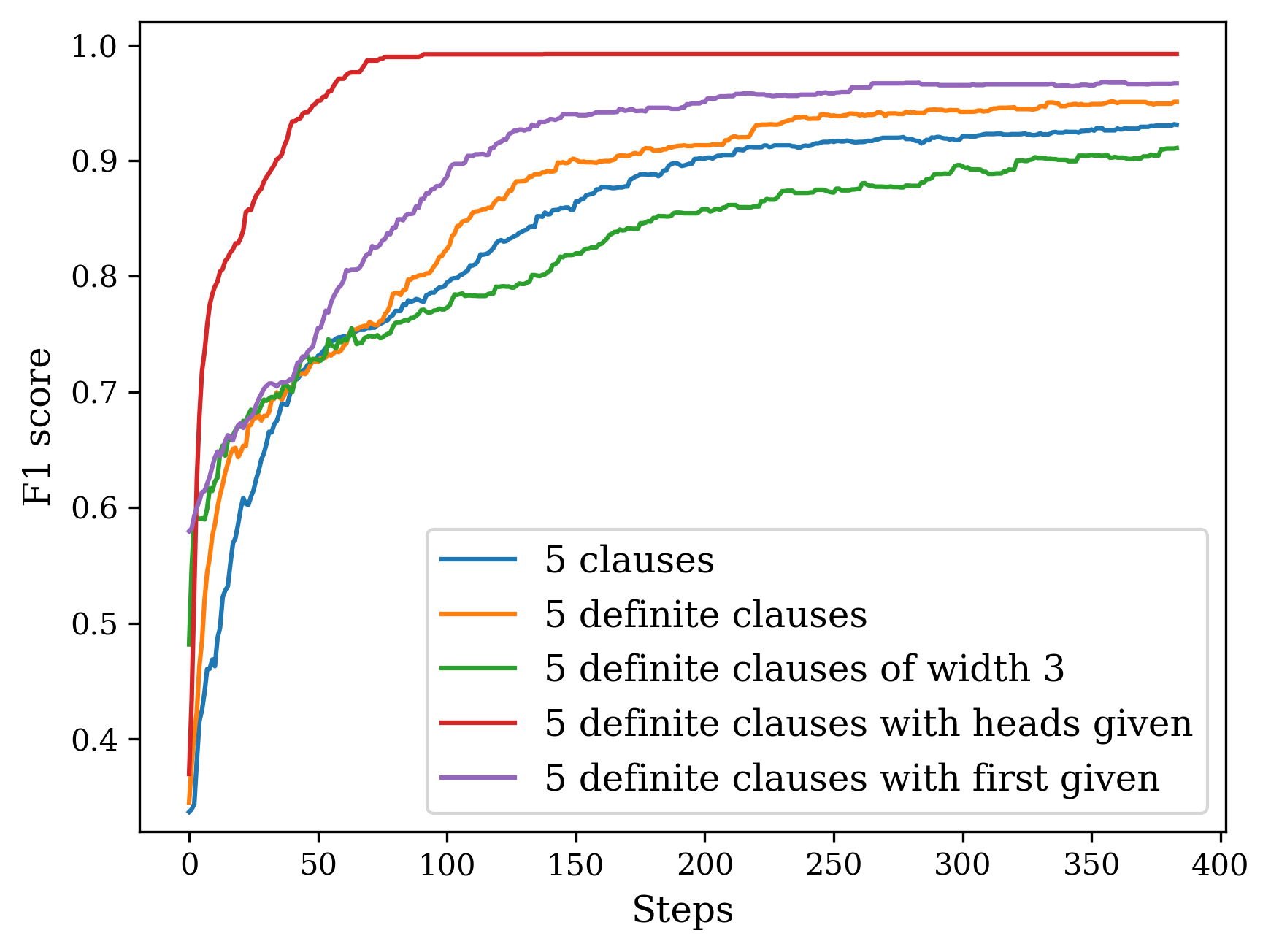}
    \caption{Average training curves, over 20 runs, of LoH formulas following different templates, for learning the same ground-truth.}
    \label{fig:definite_clauses}
\end{figure}
%\end{wrapfigure}

We can notice that adding explicit knowledge (fixing either a clause or the clause heads) yield better learning curves than the purely syntactic alternatives.
Regarding such three purely syntactic templates, the model learning definite clauses slightly outperformed the most general one, thanks to a reduced search space.   
However, despite having an even smaller hypothesis space, the model learning definite clauses of width $3$ is the one performing worse. This may be explained by the fact that the other models can update the weights of the negative literals independently, whereas updates in $[\neg v_1, \dots, \neg v_{10}]$ always involve more variables at the same time, or because the final formula is invariant to permuting the three chosen literals, but the parameterization is not. 
%or because of the asymmetry between choosing $\neg v_i$ in the first choice operator and $\neg v_j$ in the second (with $i\neq j$) and the equivalent choice of $\neg v_j$ in the first and $\neg v_i$ in the second. 
Anyway, one may still employ this model because it guarantees formulas of a prescribed template, irrespective of raw predictive performance. 
Moreover, among the $20$ runs, each of the LoH models found the $100\%$-correct formula multiple times. This suggest that trying different runs and picking the best can be a useful strategy.

\section{Datasets Properties \& Runtimes}\label{app:datasets}

For each tabular dataset, Table \ref{tab:datasets} summarizes the number of features before and after binarization, together with references and size.

\begin{table*}[!ht]
\centering
\caption{Datasets properties.}\label{tab:datasets}
\begin{tabular}{lrrrr}
\toprule
\textbf{Dataset} & \textbf{\# instances} & \textbf{\# classes} & \textbf{\makecell{\# original \\ features}} & \textbf{\makecell{\# binary \\ features}} \\
\midrule
adult  \cite{adult}       & 32561 & 2  & 14 & 155 \\
bank-marketing \cite{bank_marketing} & 45211 & 2  & 16 & 88  \\
banknote  \cite{banknote}    & 1372  & 2  & 4  & 17  \\
blogger  \cite{blogger}     & 100   & 2  & 5  & 15  \\
chess   \cite{chess}      & 28056 & 18 & 6  & 40  \\
connect-4  \cite{connect_4}   & 67557 & 3  & 42 & 126 \\
letRecog  \cite{letRecog}    & 20000 & 26 & 16 & 155 \\
magic04  \cite{magic}     & 19020 & 2  & 10 & 79  \\
mushroom   \cite{mushroom}   & 8124  & 2  & 22 & 117 \\
nursery   \cite{nursery_76}    & 12960 & 5  & 8  & 27  \\
tic-tac-toe \cite{tic-tac-toe}  & 958   & 2  & 9  & 27  \\
wine    \cite{wine}      & 178   & 3  & 13 & 37  \\
\bottomrule
\end{tabular}
\end{table*}

%\section{Runtimes}
Table \ref{tab:runtimes} reports the approximated runtime of a single training plus evaluation run, for each benchmark discussed in Section \ref{subsec:classification}. The values are relative to the selected hyperparameters, and the table also reports the corresponding models' parameter count. All experiments were conducted on a cluster node equipped with an Nvidia RTX A5000 with 60GB RAM.

\section{Hyperparameters}
\label{app:hyperparams}
The following outlines the hyperparameter search spaces for the models presented in Section \ref{sec:experiments}. The hyperparameter choices for each dataset are available in the code repository.
\begin{itemize}
    \item Decision Tree: min\_samples\_split (2--50), 
    max\_depth (2--50), 
    min\_samples\_leaf (1--50).
    \item Random Forest: n\_estimators (50--500), min\_samples\_split (2--50), 
    max\_depth (2--50), 
    min\_samples\_leaf (1--50).
    \item XGBoost: max\_depth (5--20),
    n\_estimators (10--500),
    learning\_rate ($10^{-3}$--$2\cdot 10^{-1}$).
    \item Neural Network: number of hidden layers (1--3), number of units per layer (4--128), learning rate ($10^{-4}$--$10^{-1}$). Batch size fixed at 256, and ReLU activations.
    \item DLN: number of hidden layers (1--10), number of units per layer (16--512), grad\_factor ($1.$--$2.$), learning rate ($10^{-3}$--$10^{-1}$), $\tau$ (1--100). Batch size fixed at 128.
    \item MLLP/CRS: number of hidden layers (1,3), number of units per layer (16--256), weight decay ($10^{-8}$--$10^{-2}$), learning rate ($10^{-4}$--$10^{-1}$), random binarization rate (0--0.99). Batch size (128) and learning rate scheduler were set to the default value.
    \item LoH: learning rate ($0.01$--$0.2$), Gumbel noise scale $\beta$ ($0.4$--$1.2$), temperature ($0.4$--$1.2$), and temperature rescaling factor applied every $10$ epochs ($0.9925$--$1$). Like NN and MLLP/CRS, we also optimized the architecture, allowing the TPE algorithm to select the layer type (any-clause vs fixed-size-clauses), the number of hidden layers ($1$--$2$), layer sizes ($16$--$256$), and whether the output layer is conjunctive or disjunctive (with the other layers alternating).  In case of fixed-size-clauses layer type, also the hyperparameter $k$ ($2$--$8$) was tuned for each layer.
\end{itemize}
%In appendix \ref{app:ablation}, the values of the hyperparameters  are fixed: $128$ as batch size, $0.15$ as learning rate, $1$ as temperature, and $Gumbel(0,1)$ noise.

\begin{table*}[ht]
    \centering
    \caption{Average runtimes relative to a single training + evaluation run.}
    \label{tab:runtimes}
    \begin{tabular}{l*{3}{cc}}
    \toprule
    \multicolumn{1}{l}{Dataset} &
    \multicolumn{2}{c}{DLN} &
    \multicolumn{2}{c}{MLLP/CRS} &
    \multicolumn{2}{c}{LoH}\\
    \cmidrule(lr){2-3}\cmidrule(lr){4-5}\cmidrule(lr){6-7}
    & Time (s) & \# Params & Time (s) & \# Params &  Time (s) & \# Params \\
     \midrule
        
        adult & 520 & 57680 & 499 & 74260  & 287 &  17584
        \\ 
        
        bank-marketing &  615 & 38576 & 687 & 70556 &  452 &  81612
        \\
        
        banknote & 17 & 23728 & 14 & 4807 & 12 &  7790
        \\ 
        
        blogger & 4 & 36592 & 4 & 3485 &  6 &  6195
        \\
        
        chess & 956 & 53568 & 970 & 117774 & 546 &  29464
        \\ 
        
        connect-4 & 1214 & 51344 & 650 & 25542 & 873 &  55212
        \\ 
        
        letRecog & 414 & 41280 & 511 & 126302 & 528 &  86518
        \\ 
        
        magic04 & 467 & 50960  &  271 & 19440 & 364 &  31104
        \\ 
        
        mushroom & 110  & 31024 & 78 & 16422 & 67  &  82900
        \\ 
        
        nursery & 284 & 51840 &  302 & 48331 & 311 &  68186
        \\
        
        tic-tac-toe & 25 & 18288 & 38 & 4002 & 32 & 11310 
        \\ 
        
        wine & 2 & 6960 & 5 & 6400 & 4 &  2640
        \\ 

        %Visual Tic-Tac-Toe (DNF) &  &  &  &  
        %\\
        
        %Visual Tic-Tac-Toe (CNF)  &  &  &  &  
        %\\
        \bottomrule
    \end{tabular}
\end{table*}

\section{Decision Rules of Visual Tic-Tac-Toe}
\label{app:visual_ttt}
Both CRS and LoH allow for the automatic extraction of logical formulas. In order to interpret such decision rules, we need to assign proper names to the input propositions. In particular, for visual tic-tac-toe, there are three input propositions for every image in the tic-tac-toe grid, each corresponding to an output unit of the feature-extractor CNN. The assignment of labels (such as $X$, $O$ or $B$) to such units is done by thresholding their average activations with respect to each image class ($0$, $1$, and $2$), in the following way:
\begin{itemize}
    \item if the average activation is never $>0.5$, the label for the unit is $\bot$, which can later be simplified from the formulas;
    \item if the average activation is $>0.5$ for one class and $<0.5$ for the other two, the label for the unit is the one corresponding to the activating class;
    \item if the average activation is $>0.5$ for two classes and $<0.5$ for the other, the label for the unit is the logical negation of the non-activating class;
    \item if the average activation is always $>0.5$, the label for the unit is $\top$, which can later be simplified from the formulas.
\end{itemize}
As an example, if an output unit of the CNN exhibits an average activation below $0.5$ for images of class $1$ but above $0.5$ for images of the other two classes, we assign it the label $\neg O$ (recalling that digit $1$ was associated with $O$). 
In this way, we can assign names to the input propositions of the logical models, by combining the labels of the CNN output units with the positions in the $3 \times 3$ tic-tac-toe grid. 

Table \ref{tab:visualttt_formulas} reports the \emph{best} decision rules learned by CRS and LoH on the Visual Tic-Tac-Toe task. 
The formulas were simplified removing the appearances of $\top$ and $\bot$, and also removing redundant clauses.
We do not provide the formulas learned by DLN because the implementation of DLN we used did not have a function to write them in a human-readable way. Moreover, to boost performance, DLN actually learns an ensemble with majority voting, not a single formula. The DLN run with highest Symbolic eval achieved $.885$ F$1$-score.

\begin{table*}[hb]
  \centering
  \caption{Best decision rules learned on the Visual Tic-Tac-Toe task. The labels $X_i$ and $\neg O_j$ were assigned post-hoc to each output of the trained CNN in the way explained in appendix \ref{app:visual_ttt}.} \label{tab:visualttt_formulas}
  \begin{tabular}{llcc}
    \toprule
    Model &  & Symb. eval & Formula \\
    \midrule
    \multirow{2}[1]{*}{CRS} 
        & DNF & .991 &    \small{\makecell{$(\neg O_1 \land \neg O_5 \land \neg O_9) \lor (\neg O_3 \land \neg O_5 \land \neg O_7)$ \\ $\lor (\neg O_2 \land \neg O_4 \land \neg O_7 \land \neg O_9) \lor (\neg O_1 \land \neg O_2 \land \neg O_3 \land \neg O_4 \land \neg O_7)$ \\ $\lor (\neg O_1 \land \neg O_2 \land \neg O_3 \land \neg O_4 \land \neg O_8) \lor (\neg O_1 \land \neg O_2 \land \neg O_3 \land \neg O_4 \land \neg O_9)$ \\ $\lor (\neg O_1 \land \neg O_2 \land \neg O_3 \land \neg O_5 \land \neg O_8) \lor (\neg O_1 \land \neg O_2 \land \neg O_3 \land \neg O_6 \land \neg O_7)$ \\ $\lor (\neg O_1 \land \neg O_2 \land \neg O_3 \land \neg O_6 \land \neg O_8) \lor (\neg O_1 \land \neg O_2 \land \neg O_4 \land \neg O_6 \land \neg O_7)$ \\ $\lor (\neg O_1 \land \neg O_3 \land \neg O_4 \land \neg O_7 \land \neg O_8) \lor (\neg O_1 \land \neg O_3 \land \neg O_6 \land \neg O_8 \land \neg O_9)$ \\ $\lor (\neg O_1 \land \neg O_4 \land \neg O_5 \land \neg O_6 \land \neg O_7) \lor (\neg O_1 \land \neg O_4 \land \neg O_5 \land \neg O_6 \land \neg O_8)$ \\ $\lor (\neg O_1 \land \neg O_4 \land \neg O_7 \land \neg O_8 \land \neg O_9) \lor (\neg O_1 \land \neg O_6 \land \neg O_7 \land \neg O_8 \land \neg O_9)$ \\ $\lor (\neg O_2 \land \neg O_3 \land \neg O_4 \land \neg O_5 \land \neg O_8) \lor (\neg O_2 \land \neg O_3 \land \neg O_4 \land \neg O_6 \land \neg O_9)$ \\ $\lor (\neg O_2 \land \neg O_3 \land \neg O_6 \land \neg O_7 \land \neg O_9) \lor (\neg O_2 \land \neg O_4 \land \neg O_5 \land \neg O_6 \land \neg O_8)$ \\ $\lor (\neg O_2 \land \neg O_4 \land \neg O_5 \land \neg O_6 \land \neg O_9) \lor (\neg O_2 \land \neg O_4 \land \neg O_5 \land \neg O_8 \land \neg O_9)$ \\ $\lor (\neg O_2 \land \neg O_5 \land \neg O_6 \land \neg O_7 \land \neg O_8) \lor (\neg O_2 \land \neg O_5 \land \neg O_7 \land \neg O_8 \land \neg O_9)$ \\ $\lor (\neg O_3 \land \neg O_4 \land \neg O_5 \land \neg O_6 \land \neg O_8) \lor (\neg O_3 \land \neg O_4 \land \neg O_5 \land \neg O_6 \land \neg O_9)$ \\ $\lor (\neg O_3 \land \neg O_4 \land \neg O_6 \land \neg O_8 \land \neg O_9) \lor (\neg O_3 \land \neg O_4 \land \neg O_7 \land \neg O_8 \land \neg O_9)$ \\ $\lor (\neg O_3 \land \neg O_6 \land \neg O_7 \land \neg O_8 \land \neg O_9)$
        }} \\
        \cline{2-4}
        \addlinespace[2pt]
        & CNF & .984  &  \small{\makecell{$(\neg O_1 \lor \neg O_2 \lor \neg O_3) \land (\neg O_1 \lor \neg O_4 \lor \neg O_7)$ \\$\land (\neg O_1 \lor \neg O_5 \lor \neg O_9) \land (\neg O_2 \lor \neg O_5 \lor \neg O_8)$ \\ $\land (\neg O_3 \lor \neg O_5 \lor \neg O_7) \land (\neg O_3 \lor \neg O_6 \lor \neg O_9)$ \\ $\land (\neg O_4 \lor \neg O_5 \lor \neg O_6) \land (\neg O_7 \lor \neg O_8 \lor \neg O_9)$ \\ $\land (\neg O_2 \lor \neg O_3 \lor \neg O_4 \lor \neg O_9) \land (\neg O_2 \lor \neg O_5 \lor \neg O_7 \lor \neg O_9)$}}\\ 
    \midrule
    \multirow{2}[11]{*}{LoH} 
        & DNF & 1.00 &   \small{\makecell{$(X_1 \land X_2 \land X_3) \lor (X_4 \land X_5 \land X_6)$ \\ $\lor (X_7 \land X_8 \land X_9) \lor(X_1 \land X_4 \land X_7)$  \\ $\lor (X_2 \land X_5 \land X_8) \lor  (X_3 \land X_6 \land X_9)$  \\ $\lor (X_1 \land X_5 \land X_9) \lor (X_3 \land X_5 \land X_7)$}}\\
        \cline{2-4}
        \addlinespace[2pt]
        & CNF & .999 & \small{\makecell{$(X_1 \lor X_5 \lor X_9) \land (X_3 \lor X_5 \lor X_7)$ \\ $\land (X_1 \lor X_2 \lor X_6 \lor X_7) \land (X_1 \lor X_3 \lor X_4 \lor X_8)$ \\ $\land (X_1 \lor X_3 \lor X_5 \lor X_8) \land (X_1 \lor X_3 \lor X_6 \lor X_8)$ \\ $\land (X_1 \lor X_5 \lor X_6 \lor X_7) \land (X_1 \lor X_6 \lor X_7 \lor X_8)$ \\ $\land (X_2 \lor X_3 \lor X_4 \lor X_9) \land (X_2 \lor X_4 \lor X_5 \lor X_9)$ \\ $\land (X_2 \lor X_4 \lor X_7 \lor X_9) \land (X_2 \lor X_5 \lor X_6 \lor X_7)$ \\ $\land (X_2 \lor X_5 \lor X_7 \lor X_9) \land (X_2 \lor X_6 \lor X_7 \lor X_9)$ \\ $\land (X_3 \lor X_4 \lor X_5 \lor X_8) \land (X_3 \lor X_4 \lor X_5 \lor X_9)$ \\ $\land (X_3 \lor X_4 \lor X_8 \lor X_9) \land (X_1 \lor X_2 \lor X_3 \lor X_4 \lor X_7)$ \\ $\land (X_1 \lor X_2 \lor X_3 \lor X_6 \lor X_9) \land (X_1 \lor X_2 \lor X_5 \lor X_6 \lor X_8)$ \\ $\land (X_1 \lor X_4 \lor X_5 \lor X_6 \lor X_8) \land (X_1 \lor X_4 \lor X_7 \lor X_8 \lor X_9)$ \\ $ \land (X_2 \lor X_4 \lor X_5 \lor X_6 \lor X_8) \land (X_3 \lor X_6 \lor X_7 \lor X_8 \lor X_9)$}} \\
    \bottomrule
  \end{tabular}
\end{table*}

\fi

\end{document}